\documentclass[year=26,pdfa]{fmcad}

\usepackage{microtype}
\usepackage{graphicx}
\usepackage{subcaption}
\usepackage{booktabs} 
\usepackage{float}
\usepackage{amsmath}
\usepackage{amssymb}
\usepackage{mathtools}
\usepackage{amsthm}
\usepackage{url}  
\usepackage{pgfplots} 
\usepackage{amsfonts}       
\usepackage{nicefrac}       
\usepackage{xcolor}         
\usepackage{paralist, wrapfig}
\usepackage{comment}
\usepackage{tikz}
\usepackage{multirow} 
\usepackage{multicol}
\usepackage{tikz-cd}
\usepackage{siunitx}
\usepgfplotslibrary{groupplots}
\pgfplotsset{compat=1.18}
\usetikzlibrary{shapes,arrows,positioning,calc}
\usetikzlibrary{decorations.markings,decorations.pathmorphing, calligraphy, patterns}
\usetikzlibrary{decorations.pathreplacing, backgrounds, fit}
\usepackage[capitalize,noabbrev]{cleveref}
\newif\ifextendedversion\extendedversiontrue
\usepackage{balance}
\newtheorem{definition}{Definition}
\newtheorem{problem}{Problem}
\newtheorem{assumption}{Assumption}
\newtheorem{lemma}{Lemma}
\newtheorem{theorem}{Theorem}
\usepackage{thmtools}
\usepackage{thm-restate}

\newcommand\reals{\mathbb{R}}
\renewcommand\vec[1]{\mathbf{#1}}
\newcommand\vx{\vec{x}}
\newcommand\vy{\vec{y}}
\newcommand\vz{\vec{z}}
\newcommand\va{\vec{a}}
\newcommand\vb{\vec{b}}
\newcommand\vl{\vec{\ell}}
\newcommand\vu{\vec{u}}
\newcommand\fhat{\hat{f}}
\newcommand\err{\mathsf{e}}
\newcommand\nat{\mathbb{N}}
\newcommand\psihat{\widehat{\psi}}
\newcommand\tupleof[1]{\left\langle #1 \right\rangle}

\newcommand\nn{\mathsf{n}}
\newcommand\decor[3]{#1^{(#2)}_{#3}}
\newcommand\Phihat{\widehat{\Phi}}
\newcommand\vzhat{\widehat{\vz}}
\newcommand\phihat{\widehat{\phi}}
\newcommand\lip{\Lambda}

\newcommand\MO{\mathsf{MO}}
\newcommand\ck[1]{\textcolor{blue}{CK: #1}}
\usepackage[normalem]{ulem} 
\newcommand\yhat{\hat{y}}

\newcommand\hatz{\mathbf{\hat{z}}}
\newcommand\ve{\mathbf{e}}

\newcommand\relu{\mathsf{RLu}}

\begin{document}
\title{Optimized Piecewise Affine Abstractions of Neural Networks with Learnable Activation Functions}
\author{%
\IEEEauthorblockN{%
Noah Schwartz*\IEEEauthorrefmark{2}\,,
Chandra Kanth Nagesh*\IEEEauthorrefmark{2}\,\orcid{0000-0003-2122-396X},
Sriram Sankaranarayanan\IEEEauthorrefmark{2}\,\orcid{0000-0001-7315-4340}\\
Ramneet Kaur\IEEEauthorrefmark{3},
Tuhin Sahai\IEEEauthorrefmark{3},
Susmit Jha\IEEEauthorrefmark{4}}
\IEEEauthorblockA{%
\IEEEauthorrefmark{2} University of Colorado Boulder, Boulder, CO, USA. Email: first.lastname@colorado.edu}
\IEEEauthorblockA{%
\IEEEauthorrefmark{3}SRI International, Menlo Park, CA, USA. Email: first.lastname@sri.com}
\IEEEauthorblockA{%
\IEEEauthorrefmark{4}Defense Advanced Research Projects Agency (DARPA), Arlington, VA, USA. Email: first.lastname@darpa.mil}
}

\maketitle

\begin{abstract}
We present a generalized framework for the range verification of neural networks featuring non-linear activation functions. Our approach first constructs an ``optimized piecewise affine abstraction" of the network that replaces each non-linear activation function by a piecewise affine (PWA) function plus a bounded error. Such PWA functions are readily amenable to existing neural network verification techniques using specializations of linear arithmetic SMT solvers and mixed-integer optimization approaches. However, there are infinitely many ways to abstract each node, with a natural tradeoff between the number of pieces used, the global error bound, and the complexity of the resulting verification problem. We propose a dynamic programming (DP) algorithm to systematically compute the optimized PWA abstraction for general activation functions, guaranteeing tighter output bounds. The algorithm combines a local DP approximation at each node with a global error bound, yielding a variant of the knapsack problem for deciding how to allocate a fixed budget on the total number of pieces across units so as to minimize the worst-case error bound between the network and its approximation. Although the knapsack problem is itself NP-hard, we can use pseudo-polynomial DP algorithms as well as approximation schemes to solve it efficiently. Crucially, our approach is broadly applicable to diverse networks consisting of non-linear activations, including standard Multi-Layer Perceptrons (MLPs) and recently proposed architectures such as Kolmogorov-Arnold Networks (KANs). Over a series of KAN benchmarks spanning 20 to 22,000 parameters, our approach yields output bounds that are consistently of smaller width than uniform PWA allocation. The overall time taken is roughly comparable while the overhead for computing the optimized abstraction is subsumed by the time taken to compute output bounds.
\end{abstract}

\IEEEpeerreviewmaketitle

\section{Introduction}\label{sec:intro}

We examine piecewise affine (PWA) abstractions for solving range verification problems for feedforward neural networks whose nodes (neurons) feature a wide class of nonlinear and ``learnable'' activation functions.  Range verification involves computing bounds on the outputs of the network, given a set of inputs~\cite{Liu+Others/2021/Algorithms,albarghouthi-book}. Range verification serves as a primitive for verifying properties of neural networks for applications such as control of safety-critical autonomous systems~\cite{Irfan+Others/2020/Towards,katz_barrett_dill_julian_kochenderfer_2017} and robustness of networks to input perturbations~\cite{tjeng_xiao_tedrake_2017,Casadio+Others/2022/NEural}. Commonly used neural networks typically employ activation functions such as sigmoid, ReLU, tanh, Leaky ReLU and GeLU~\cite{DUBEY202292}. These functions are fixed as part of the network architecture and typically do not change as the network is trained. Recently proposed network models such as \emph{deep spline networks}~\cite{Pakshal+Others/2020/Learning,Aziznejad+Others/2020/Deep} and Kolmogorov-Arnold Networks~\cite{li2024kolmogorovarnold,Liu+Others/2025/KAN} employ ``free-form'' activation functions modeled using B-splines or radial basis functions which can change during training time in order to better fit the data.
As a result, such networks have highly nonlinear activation functions  that are not fixed \emph{a priori}. At the same time, commonly used approaches for solving  neural network verification problems have relied on the activation functions of the nodes in the network being piecewise affine (PWA) functions~\cite{katz_barrett_dill_julian_kochenderfer_2017,tjeng_xiao_tedrake_2017} or being abstracted by a PWA function with known error bounds. 
However, the strategy for approximating activation functions by PWA functions is often fixed a priori without consideration of the specific neural network instance.

In this paper, we observe that the choice of a PWA abstraction is, in itself, a decision problem: (a) how many pieces do we use to approximate a given unit in the NN? (b) how does our choice of pieces affect the overall error between the function represented by the NN and its PWA approximation? and (c) how do we optimally choose the number of pieces at each unit locally so that the global error estimate can be minimized?  Our framework proposed in this paper provides answers to these questions through dynamic programming. 

\begin{figure*}[t]
    \centering
    \resizebox{0.6\linewidth}{!}{\begin{tikzpicture}[ node distance = 0.75cm,
  phi/.style={rectangle, draw, text centered },
  summ/.style={circle, draw, fill=red!20, text centered},
  input/.style={circle, draw, fill=black, minimum size=0.3pt, inner sep=0.3pt},
  connection/.style={draw, -latex'}
]

\begin{scope}
    \node at (-6.5,0.75)[draw=black, fill=yellow!20]{ $1$};
    \begin{scope}[xshift=-4.5cm]
    \draw (-1.2, -0.2) rectangle (1.2,1.2);
    \draw [<->, thick] (0, -0.2) -- (0,1.2);
    \draw [<->, thick] (-1.2, 0) -- (1.2,0);
    \draw[blue, ultra thick] plot [smooth, tension=1] coordinates {(-1,0) (-0.8, 0.2) (-0.5, 0.5) (0.0, 1.0) (0.5, 0.7) (0.7, 0.3) (1.0, 0.0) };
    \end{scope}

    \begin{scope}[xshift=-1cm, yshift =0.5cm]
    \draw (-1.2, -0.2) rectangle (1.2,1.2);
    \draw [<->, thick] (0, -0.2) -- (0,1.2);
    \draw [<->, thick] (-1.2, 0) -- (1.2,0);
    \draw[line width=10pt, draw=blue!10, opacity=0.5] (-1, 0) -- (0.1, 1.0) -- (1,0);
    \draw[blue, ultra thick, dashed] plot [smooth, tension=1] coordinates {(-1,0) (-0.8, 0.2) (-0.5, 0.5) (0.0, 1.0) (0.5, 0.7) (0.7, 0.3) (1.0, 0.0) };

    \draw[thick, red] (-1, 0) -- (0.1, 1.0) -- (1,0);

    \node at (0,-.4) { 2 pieces};
    \node at (0,-.7) { error $= 4.0$};

    \end{scope}

    \begin{scope}[xshift=1.6cm, yshift =0.5cm]
    \draw (-1.2, -0.2) rectangle (1.2,1.2);
    \draw [<->, thick] (0, -0.2) -- (0,1.2);
    \draw [<->, thick] (-1.2, 0) -- (1.2,0);
    \draw[line width=5pt, draw=blue!10, opacity=0.5] (-1, 0) -- (-0.2, 0.9) -- (0.2, 1.0) --  (1,0);
    \draw[blue, ultra thick, dashed] plot [smooth, tension=1] coordinates {(-1,0) (-0.8, 0.2) (-0.5, 0.5) (0.0, 1.0) (0.5, 0.7) (0.7, 0.3) (1.0, 0.0) };

    \draw[thick, red] (-1, 0) -- (-0.2, 0.9) -- (0.2, 1.0) --  (1,0);

    \node at (0,-.4) { 3 pieces};
    \node at (0,-.7) { error $=1.2$};

    \end{scope}

    \begin{scope}[xshift=5cm, yshift =0.5cm]
    \draw (-1.2, -0.2) rectangle (1.2,1.2);
    \draw [<->, thick] (0, -0.2) -- (0,1.2);
    \draw [<->, thick] (-1.2, 0) -- (1.2,0);
    \draw[line width=2pt, draw=blue!10, opacity=0.5]  (-1,0) --  (-0.8, 0.2) -- (-0.5, 0.48) --  (-0.2, 0.9) -- (0.0, 1.0) -- (0.1, 1.03) --  (0.3, 0.9) -- (0.5, 0.7) -- (0.7, 0.3) -- (1.0, 0.0);
    \draw[blue, ultra thick, dashed] plot [smooth, tension=1] coordinates {(-1,0) (-0.8, 0.2) (-0.5, 0.5) (0.0, 1.0) (0.5, 0.7) (0.7, 0.3) (1.0, 0.0) };
    \draw[thick, red] (-1,0) --  (-0.8, 0.2) -- (-0.5, 0.48) --  (-0.2, 0.9) -- (0.0, 1.0) -- (0.1, 1.03) --  (0.3, 0.9) -- (0.5, 0.7) -- (0.7, 0.3) -- (1.0, 0.0);

    \node at (0,-.4) { 10 pieces};
    \node at (0,-.7) { error $=0.02$};

    \end{scope}

    \node at (3.35,0.5) {\large $\cdots$};
    \draw [decorate, line width=1pt, decoration = {calligraphic brace}] (-2.5,-0.3) -- (-2.5,1.7);
    \draw [decorate,line width=1pt, decoration = {calligraphic brace}] (6.5,1.7) --  (6.5,-0.3);

    \begin{pgfonlayer}{background}
    \draw[draw=black, thin] (-6.5,2) rectangle (7, -0.5);
    \end{pgfonlayer}
\end{scope}

\begin{scope}[yshift=-3.0cm]
    \node at (-6.5,0.75)[draw=black, fill=yellow!20]{ $2$};

    \node[input](n0i)  at (-6, 2) {}; 
    \node[input, below=1cm of n0i](n1i)  {}; 
    \node[input, below=1cm of n1i](n2i)  {}; 
    \node[input, below=1cm of n2i](n3i)  {}; 

    \node[input, right=0.6cm of n0i](n0)   {}; 
    \node[input, right=0.6cm of n1i](n1)  {}; 
    \node[input, right=0.6cm of n2i ](n2)  {}; 
    \node[input, right=0.6cm of n3i](n3)  {}; 

    \path[->] (n0i) edge node[above]{$x_1$} (n0)
    (n1i) edge node[above]{$x_2$} (n1)
    (n2i) edge node[above]{$x_3$} (n2)
    (n3i) edge node[above]{$x_4$} (n3);

    \node[phi, right=0cm of n0, name=phi11]{\scriptsize $\decor\psi{1}{1}$};
    \node[phi, right=0cm of n1, name=phi12]{\scriptsize $\decor\psi{1}{2}$};
    \node[phi, right=0cm of n2, name=phi13]{\scriptsize $\decor\psi{1}{3}$};
    \node[phi, right=0cm of n3, name=phi14]{\scriptsize $\decor\psi{1}{4}$};

    \node[summ, below right=0cm and 0.6cm  of phi11, name=s1]{$\sum$};
    \node[summ, below right=0cm and 0.6cm  of phi13, name=s2]{$\sum$};

    \path[->] (phi11) edge (s1) 
    (phi12) edge (s1)
    (phi13) edge (s2)
    (phi14) edge (s2); 

    \node[phi, right=0.3cm of s1, name=phi21]{\scriptsize $\decor\psi{2}{1}$};
    \node[phi, right=0.3cm of s2, name=phi22]{\scriptsize $\decor\psi{2}{2}$};
    \node[summ, below right=0.5cm and 0.3cm of phi21, name=s3]{$\sum$};
    \node[phi, right=0.3cm of s3, name=phi31]{\scriptsize $\decor\psi{3}{1}$};
    \node[input, name=o1, right=0.3cm of phi31]{};

    \path[->] (s1) edge (phi21)
    (s2) edge (phi22)
    (phi21) edge (s3)
    (phi22) edge (s3)
    (s3) edge (phi31)
    (phi31) edge node[above]{$y$} (o1);
    \draw [decorate, line width=1pt, decoration = {calligraphic brace}] (0.5,2) --  (0.5,-1);
    \node at (3.7,0.5) { \large $ \begin{array}{rcl} 
    \err & = &  \max_{\vx \in \reals^4} \ | f_N(\vx) - \fhat_N(\vx) | \\ 
    & \leq  & 2.5 \decor\err{1}{1} + 3 \decor\err{1}{2} + \decor\err{1}{3} +  \decor\err{1}{4} \\
    && + 1.5 \decor\err{2}{1} + 0.6 \decor{\err}{2}{2} + 0.2 \decor\err{3}{1} \\ 
    \end{array}$};

    \begin{pgfonlayer}{background}
    \draw[draw=black, thin] (-6.5,-1.5) rectangle (7, 2.5);
    \end{pgfonlayer}
\end{scope}

\begin{scope}[yshift=-6.7cm]
    \node at (-6.5,0.75)[draw=black, fill=yellow!20]{ $3$};
    \node at (-4.3, 1.7){\scriptsize
    $ \begin{array}{|c|ccc|}  
    \hline 
    \# \text{piece} & 1 &  \cdots & 20 \\  
    \decor\err{1}{1} & 5.0 & \cdots & 0.02\\ 
    \hline
    \end{array}$
    };

    \node at (-4.3, 0.9){\scriptsize
    $  \begin{array}{|c|ccc|}  
    \hline 
    \# \text{piece} & 1 &  \cdots & 20 \\  
    \decor\err{1}{2} & 3.2 &  \cdots & 0.01\\ 
    \hline
    \end{array}$
    };

    \node at (-4.3, 0.4){$\cdots$}; 

    \node at (-4.3, -0.1){\scriptsize
    $  \begin{array}{|c|ccc|}  
    \hline 
    \# \text{piece} & 1 &  \cdots & 20 \\  
    \decor\err{3}{1} & 1.7 &  \cdots & 0.1\\ 
    \hline
    \end{array}$
    };

    \draw [decorate,line width=1pt, decoration = {calligraphic brace}] (-2.3,1.8) --  (-2.3,-0.3);    
    \node[input](i0) at (-2.2, 0.75){};

    \node[rectangle, inner sep=5pt, draw=black, right=0.55cm of i0](n0){\begin{tabular}{c}
    \large\textsc{Knapsack}
    \end{tabular}};

    \node[above=0.5cm of n0, name=i1]{$\delta$};
    \node[input, below=0.7cm of n0, name=i2]{};
    \path[->] (i0) edge (n0)
    (i1) edge (n0)
    (n0) edge (i2);

    \node[below=0cm of i2, name=tab1]{
    \scriptsize
    $ \begin{array}{|c|ccc|}  
    \hline 
    \text{node} & \decor{\psi}{1}{1} &   \cdots & \decor\psi{3}{1} \\ 
    \text{\# pieces} & 5 & \cdots & 6 \\ 
    \hline 
    \end{array}$
    };

    \node at (2,0.75)[draw=black, fill=yellow!20]{ $4$};

    \node at (4.5, 0.4){ $\begin{array}{l}
    \max/\min \;\; y \\ 
    \mathsf{s.t} \\
    \; \text{Constraints for PWA model} \\ 
    \; \text{ + approximation error.}
    \end{array}$  };

    \begin{pgfonlayer}{background}
    \draw[draw=black, thin] (-6.5,-1.4) rectangle (1.7, 2.2);
    \draw[draw=black, thin] (2,-1.4) rectangle (7, 2.2);
    \end{pgfonlayer}

\end{scope}
\end{tikzpicture}}
    \caption{An illustration of the key steps of our approach. 1. We consider multiple possible choices for each unit/activation function; 2. We upper bound the overall network-wide approximation error as a linear combination of individual errors; 3. We solve a knapsack problem to choose an optimized approximation that minimizes the error bound; and 4. We connect our PWA abstraction to existing NN verification solvers, especially the mixed integer approach.}
    \label{fig:approach-flow}
\end{figure*}

The key steps of our approach are illustrated in Fig.~\ref{fig:approach-flow}. Let $f_N: \reals^n \rightarrow \reals$ be the function defined by a NN. We assume single output for the ease of presentation. We construct a PWA approximation $\fhat_N: \reals^n \rightarrow \reals$ as follows: 

\begin{compactenum}
\item We consider a set of possible choices for approximating each unit activation function $\psi$ in the network by  PWA functions $\psihat_1, \ldots, \psihat_P$ with increasing number of pieces and decreasing error. We adapt a classic DP algorithm first proposed by Bellman and Roth for this purpose~\cite{Bellman+Roth/1969/Curve,Bellman/1961/Approximation}. We construct a ``trade-off'' table for each unit based on the error $\err_j = \max_{z} | \psi(z) - \psihat_j(z)|$, as the number of pieces $j$  varies.
\item \emph{Error Analysis:} We bound the worst case error $\err_{\max} = \max_{ \vx \in X} |f_N - \fhat_N|$ between the NN and its PWA approximation as a function of the errors made at each individual unit. Theorem~\ref{thm:kan-overall-error} of the paper shows that this error bound can be written as a weighted linear combination of the error bounds obtained for each unit. 
\item \emph{Knapsack Formulation:}  We find a surprising connection between the problem of finding an ``optimized'' abstraction and solving a variant of the \emph{knapsack} problem: a combinatorial optimization problem that involves the optimal choice of items that maximize value while constraining total weights to fall below a budget. 
\item \emph{Connecting Back to Known Activation Functions:} Having identified an optimized allocation of the number of pieces and error at each node of the network, we prove that any PWA function can be written as a weighted combination of ReLU units (Theorem~\ref{thm:relu-encoding}). This allows us to connect our PWA abstraction to a host of existing NN verification techniques that specialize to ReLU networks. 
\item \emph{MILP Formulation:} We solve the range verification problem, using  Mixed Integer Linear Programming (MILP), while encoding the error bounds for the PWA approximation  in the MILP model. This allows the verification problem to be translated into a MILP, which can be solved through powerful combinatorial optimization solvers such as Gurobi~\cite{gurobi}, CPLEX~\cite{cplex} and MOSEK~\cite{mosek}. In fact, by translating a PWA function into a weighted sum of ReLU units, our approach may in fact integrate well within existing NN verification tools.

\end{compactenum}

We focus our empirical evaluation on a recently proposed class of feedforward networks with learnable activation functions called Kolmogorov Arnold Networks (KANs)~\cite{li2024kolmogorovarnold,Liu+Others/2025/KAN} that have become popular across many applications in scientific machine learning. Our experiments on KANs trained for four tasks of increasing scale demonstrate that the upfront cost of the knapsack-based abstraction is recovered many times over during verification. {Over a series of KAN benchmarks spanning 20 to 22,000 parameters, our approach yields output bounds that are consistently of smaller width than uniform PWA allocation. The overall time taken is roughly comparable while the overhead for computing the optimized abstraction is subsumed by the time taken to compute output bounds.} However, we  note that the overhead of computing the optimized abstraction can be  significant, especially for smaller networks, when compared to the time taken to solve the overall range verification problem. That said, the abstraction can be solved once and reused for multiple verification instances. This is often the case for applications such as certifying robustness of neural networks against adversarial perturbations wherein the certification is performed over ranges that arise from a large set of test examples~\cite{SinghEtAl2019AbstractDomain}.
Also, the use of the optimized abstraction shows gains in the computation time for the verification problem in many instances even when the abstraction overhead is taken into account. 

Due to space limitations, detailed proofs of some results and details of the benchmarks used are presented in appendices that have been posted in Arxiv~\cite{Schwartz+Others/2026/Optimized}. 


\section{Related Work}\label{sec:related}

\subsection{Verification Problems and Their Significance}

This paper focuses on verification problems for neural networks and, in particular, the range analysis problem.  Neural networks are now used in safety-critical applications such as self-driving cars \cite{bojarski_testa_2016}, robotics \cite{gu2017deep}, and smart prosthetics \cite{clark_campbell_amor}. In this context, they suffer from some unique challenges beyond incorrect decisions arising from lack of training data and over-fitting. Adversarial attacks were first reported by Goodfellow et al.~\cite{goodfellow_shlens_szegedy_2015}, where small input perturbations on datasets like MNIST-10 (handwritten digits) led to wrongful classifications with high probability. Since then, several techniques have been developed to train models to prevent such attacks. A detailed survey can be found here \cite{bai2021recent,silva2020opportunities}.

Verification tools have  addressed these issues by proving that a given neural network is robust to certain types of perturbations or synthesizing a counterexample that shows a vulnerability in the network. Existing tools have been focused on feedforward neural networks. They  include Reluplex~\cite{katz_barrett_dill_julian_kochenderfer_2017}, \textsc{Marabou} (SMT-based) \cite{katz_huang_2019}, \textsc{CROWN} (uses convex relaxations) \cite{zhang_weng_chen_hsieh_daniel_2018},  \textsc{Sherlock} (MILP-based) \cite{dutta_chen_jha_sriram_sankaranarayanan_tiwari_2019,Dutta+Others/2018/Output} and others~\cite{lomuscio_maganti_2017} (see~\cite{Huang2020Survey,Liu+Others/2021/Algorithms,albarghouthi-book} for a detailed survey). These tools reduce the complex verification question to  range verification, the precise problem being tackled in this paper. 
They have been used to certify safety of an aircraft collision avoidance system~\cite{Irfan+Others/2020/Towards,katz_barrett_dill_julian_kochenderfer_2017}, prove bounds on inputs to control systems~\cite{Majid+Others/2022/Safe}, study the sensitivity of cyber-physical models to perturbations~\cite{Kushner+Sankaranarayanan+Breton/2020/Conformance} and certify robustness to certain types of adversarial attacks~\cite{tjeng_xiao_tedrake_2017,Singh+Others/2019/Beyond}. 

\subsection{NNs with Learnable Activation Functions}

Neural networks with  learnable ``free form'' activation functions such as deep-spline networks have been shown to be advantageous over those with fixed activation functions \cite{Aziznejad+Others/2020/Deep,Pakshal+Others/2020/Learning} in terms of being able to learn more succinct and accurate function representations but at the cost of increased training time. More recently,  Kolmogorov Arnold Networks (KANs) also replace fixed nonlinear activation functions with general univariate functions that can change during the training process. This leads to networks that are more parameter-efficient than MLPs, while still maintaining competitive or superior performance~\cite{Liu+Others/2025/KAN}.  KANs have been shown to achieve better accuracies when compared with larger MLPs for the same learning task~\cite{howard_jacob_murphy_heinlein_stinis_2024,toscano_wang_karniadakis_2024}. Applications  include time-series predictions~\cite{genet2024tkan}, vision \cite{mahara2024dawn} and physics-based learning \cite{abueidda2025deepokan}. Due to the nonlinear and free-form nature of these activation functions, these networks are not amenable to existing range verification approaches beyond those based on simple interval bound propagation~\cite{moore2009introduction}, which suffers from the well-known wrapping effect and consequently tends to yield very coarse bounds on networks of this kind.

\subsection{Neural Network Abstraction}

 Abstracting a neural network with activation functions such as tanh, sigmoid and gelu by means of a PWA function has been investigated as part of most range verification approaches~\cite{albarghouthi-book}. 
 Pulina and Tacchella  were one of the first to propose a form of piecewise affine abstractions by replacing each nonlinear activation function by a series of intervals~\cite{Pulina+Tacchella/2012/Challenging}. Their abstraction of nonlinear activation functions subdivides the domain uniformly using a single width parameter $p$ and uses the fact that commonly used activation functions (sigmoid and tanh) are monotonically increasing functions to bound the value of the function in each subinterval. If a given large value of the width parameter is unable to prove the property, the refinement process lowers the width to produce a finer grained abstraction.  In contrast, our work here uses a PWA approximation that is optimized in two senses: we use Bellman's algorithm at each node to find the best way to grid the domain into $k$ pieces that minimizes the approximation error~\cite{Bellman/1961/Approximation,Bellman+Roth/1969/Curve}, and we use a different number of pieces at different nodes. While refinement is not part of the approach in this paper, we can implement a refinement loop that is similar in spirit to the work of Pulina and Tacchella by incrementally increasing the total number of pieces used in our abstraction. 
 
 Following Pulina and Tacchella, most subsequent approaches either focus exclusively on networks with piecewise affine activation functions such as ReLU~\cite{lomuscio_maganti_2017,katz_barrett_dill_julian_kochenderfer_2017,Ehlers/2017/Formal,Tran+Othrs/2019/Star}, or employ an \emph{a priori} fixed abstraction that is performed once for each activation function to achieve a desired local error bound (see for example~\cite{dutta_chen_jha_sriram_sankaranarayanan_tiwari_2019}). Accounting for this error yields a sound output range that can be used to establish properties. While such approaches can work for activation functions such as sigmoid and tanh that are well approximated by PWA functions with up to $5$ pieces, they fail for networks that employ B-splines or radial basis functions as their activation functions. 

 Bound propagation approaches such as CROWN~\cite{zhang_weng_chen_hsieh_daniel_2018} or DeepPoly~\cite{Singh+Others/2019/Beyond} can eliminate the need for an upfront abstraction by generating upper and lower bounds for nonlinear activation functions on-the-fly as the analysis is performed. At the same time, such approaches can yield coarse bounds if the input bounds are large and the activation function is not monotone. The recent work of  Shi et al.~\cite{Shi+Others/2025/Neural} improves the handling of nonlinear activation functions such as tanh, sinh and variants of ReLU using a generalized branch-and-bound approach that iteratively branches nonlinear activation functions at carefully chosen points, and uses the bound propagation approach of alpha-beta-CROWN~\cite{Wang+Others/2021/BetaCROWN}. Shi et al.\ assume a pre-optimized set of branching points for their approach for a fixed library of commonly occurring nonlinear functions and choose a single branching point on the fly, whereas our approach computes a PWA approximation upfront for a given error tolerance $\delta$ through dynamic programming. In a sense, the contributions of our work can likely improve the approach of Shi et al.\ by providing a principled approach for choosing branching points.

 The work of Ivanov et al.\ uses a creative approach that encodes neural networks as nonlinear differential equations and converts neural network verification problems to hybrid systems verification problems~\cite{IvanovEtAl2019Verisig}. While activation functions such as tanh and sigmoid can indeed be seen as satisfying a differential equation, it is not clear how arbitrary activation functions can be encoded in this framework. 

\subsection{Verification for KANs and B-Spline Networks}
To the best of our knowledge, few publications have explored the verification of KANs, and the field remains relatively unexplored. The related problem of robustness of  KANs to adversarial perturbation has been investigated: Heiderich et al.~\cite{heiderich2025training} present a comparative analysis of KANs and MLPs in the context of adversarial robustness, concluding that standard KAN architectures are similarly vulnerable and introducing a new training method.
\cite{polomolina2024monokancertifiedmonotonickolmogorovarnold} modify the KAN architecture by using Cubic Hermite splines in place of B-Splines to guarantee monotonicity: the range analysis problem for such networks is solved simply by evaluating the network at two corner-points of the input hyper-rectangle.  Additional efforts include the works of \cite{shen2025reduced} and \cite{zeng2024kan}, where they similarly highlight the importance of adversarial robustness and study techniques to improve the effectiveness of KANs against added noise. Our approach is the first to consider the problem of systematically computing PWA abstractions of KANs using a reduction to a variant of the knapsack problem. 

\section{Preliminaries}\label{sec:range-analysis-problem}

In this section, we will formally define the neural network model and provide some key assumptions on the activation functions so that they can be approximated by a PWA function with finitely many pieces. 

\noindent\textit{Notation:} We will represent vectors $\vec{x}, \vec{y}, \vec{z}$ using boldface notation. For a vector $\vec{z}$ its $j^{th}$ component will be written $z_j$. For vectors $\vx, \vy \in \reals^n$, we define $\vz = \vx \odot \vy$ (the Hadamard product) of two vectors as the $n$ dimensional vector such that $z_i = x_i y_i$.

\begin{definition}\label{def:general-nn}
A general feedforward neural network with $K \geq 0$ hidden layers with $n$ inputs and $m$ outputs is defined by 
\begin{compactenum}
\item Weight matrices $W_1, \ldots, W_{K}, W_{K+1}$,  wherein each $W_i$ has dimensions $n_i \times n_{i-1}$ with $n_0 = n$ and $n_{K+1} = m$, and
\item Layerwise activation functions $\sigma_1, \ldots, \sigma_K$ wherein $\sigma_{i}: \reals^{n_{i}} \rightarrow \reals^{n_i}$ defines the nonlinear transformation at each hidden layer, and is itself a vector of $n_i$ individual unit-level activation functions $\psi_{i,j}: \reals \rightarrow \reals$.
\[ \sigma(\vec{z}) = \left(\begin{array}{c} 
\psi_{i,1}(z_1) \\ 
\vdots \\ 
\psi_{i, n_i}(z_{n_i}) \\ 
\end{array}\right)\]
\end{compactenum}
The function computed by the network is given by 
\[ f_N(\vec{x}) = W_{K+1} \times \sigma_K( W_K \times \cdots \times \sigma_2( W_2 \times \sigma_1( W_1 \times \vec{x}))) \,.\]
\end{definition}

\begin{assumption}[Neural Network Input Assumption]\label{assum:nn}
We assume that the inputs to the neural network fall inside a set $\vx \in X$ wherein $X = [-M_1, M_1] \times \cdots \times [-M_n, M_n]$ for fixed (but possibly large) bounds  $M_1, \ldots, M_n \geq 0$.
\end{assumption}

\begin{assumption}[Assumptions on $\psi$]\label{assum:psi}
    We will assume that each unit level activation function 
    $\psi: \reals \rightarrow \reals$ satisfies the following:
    \begin{compactenum}
	\item $\psi(z)$ is continuous and differentiable.
	\item  There is a procedure $\textsc{DerivativeInterval}(\psi, a, b)$ that given input interval $z \in [a, b]$, outputs an interval $[\ell, u]$ such that 
    $[\ell, u] \supseteq \left\{ \frac{d\psi}{dz}\ |\ z \in [a, b]\right\}$. Also, there exists an error tolerance $\epsilon > 0$ such that:
	\begin{inparaenum}
		\item  $\ell \leq \min_{z \in [a, b]} \frac{d\psi}{dz} \leq \ell + \epsilon$,
		\item $u - \epsilon \leq \max_{z \in [a,b]} \frac{d\psi}{dz} \leq u$.
	\end{inparaenum}
    The \textsc{DerivativeInterval} procedure will help us sample derivatives of $\psi$ at discrete points and reason about how the function behaves between these discrete points. 
    \end{compactenum}
\end{assumption}

\begin{lemma}\label{lemma:unit-input-bound}
As a consequence of Assumptions~\ref{assum:nn} and ~\ref{assum:psi}, for each unit $\psi_{i,j}$ there exists a constant $L_{i,j} \geq 0$ such that if the input to the network is $\vx \in X$, then the input to the unit lies in the interval  $z \in [-L_{i,j}, L_{i,j}]$.
\end{lemma}
\begin{proof}
Proof is a simple consequence of the continuity of each activation function which guarantees that the image of a compact (closed and bounded) set remains compact. 
\end{proof}

In practice, the domain $[-L_{i,j}, L_{i,j}]$ for each unit in the network can be over-approximated by a standard interval analysis carried out over the entire network~\cite{DBLP:journals/corr/abs-1810-12715}. 

These assumptions are satisfied for commonly used activation functions including gelu, tanh, sigmoid as well as learnable activation functions including  B-Splines~\cite{Liu+Others/2025/KAN,torchkan}, Pad\'e approximations~\cite{afzal_2024}, Chebyshev polynomials \cite{ss_ar_r_kp_2024}, Gaussian processes~\cite{Chen/2024/Gaussian}, and Fourier Series~\cite{Xu+Others/2024/FourierKAN}.

\begin{definition}[Continuous PWA Function]\label{def:cpwa-function}
A $k$-piece univariate, continuous PWA function (for $k \geq 2$) is represented by a sequence of $k-1$ breakpoints: $\tupleof{(z_1,y_1) \ldots, (z_{k-1}, y_{k-1})}$ and end slopes $m_0, m_{k-1}$. The breakpoints satisfy the ordering constraint wherein $ z_1 <  \cdots < z_{k-1} $ such that, the PWA function $\psihat$ is defined as:
\[ \psihat(z) = \begin{cases}
	y_i + (z - z_i)\frac{y_{i+1} - y_i}{z_{i+1} - z_i} &  \text{if}\ z \in [z_i, z_{i+1}] \\ 
    y_1 + m_0 (z - z_1) & z \leq z_1 \\ 
    y_{k-1} + m_{k-1} (z - z_{k-1}) & z \geq z_{k-1} \\ 
    \end{cases}
\]
where, $i \in \{ 1, \ldots, k-2\}$. Let $m_i := \frac{y_{i+1} - y_i}{z_{i+1} - z_i}$ for $i  \in \{ 1, \ldots, k-2\}$.
\end{definition}

\paragraph{Range Analysis Problem }
The range analysis problem seeks a guaranteed range for the outputs of a network given a set of inputs. Let $f_N: \reals^{n} \rightarrow \reals$ be a NN with $K$ layers as defined in Def.~\ref{def:general-nn}. Suppose we define a range $\vx \in [\vl, \vu] \subseteq X$, wherein $\vl, \vu \in \reals^n$ are vectors and $\vx \in [\vl, \vu]$ denotes the box constraint $\ell_i \leq x_i \leq u_i$ for each component $i \in \{1, \ldots, n\}$. Our goal is to find a ``tight'' range that over-approximates all the possible output values. In other words, we wish to find $\alpha, \beta \in \reals$ with $\alpha \leq \beta$ such that
\begin{equation} \label{Eq:range-analysis-guarantee}
	\forall\ \vx \in [\vl, \vu], f_N(\vx) \in [\alpha, \beta]
\end{equation}

Although $\alpha = -\infty$ and $\beta = \infty$ is a trivial solution, we often seek to find the tightest possible bounds that establish that $[\alpha, \beta] \subseteq [a, b]$ for some user-provided correctness property
\[ \vx \in [\vl, \vu] \Rightarrow f_N(\vx) \in [a, b]\,.\]

\begin{problem}[Range Verification Problem]\label{def:problem-stmt}
Given a KAN $N$ computing a function $f_N: \reals^n \rightarrow \reals$, an input range $\vx \in [\vl, \vu]$, we wish to find  the ``tightest-possible'' range over the output $[\alpha, \beta]$ satisfying Eqs.~\eqref{Eq:range-analysis-guarantee}.
\end{problem}

We can extend Problem~\ref{def:problem-stmt} to  cases wherein the inputs $\vx$ range over a polyhedron $P[\vx]$ since our approach in this paper is based on a MILP formulation and the polyhedral constraint over the inputs can be directly encoded.


\section{Optimized Piecewise Affine (PWA) Approximations for Univariate Functions}\label{sec:optimal-pwa-univariate-fun}

Let us consider a univariate function $\psi(z) = \psi_{i,j}(z)$, which appears as part of the neural network from Def.~\ref{def:general-nn},  satisfying Assumption~\ref{assum:psi}. We will provide an approach to approximate it using a PWA function with a fixed bound on the number of pieces within the interval $z \in [-L, L]$ wherein $L=L_{i,j}$ corresponds to the input domain of this particular unit when the network input domain is the set $X$ (see Assumption~\ref{assum:nn} and Lemma~\ref{lemma:unit-input-bound}).

\paragraph{Domain Discretization} Since the function $\psi$ is not known to us, we will discretize the continuous range $z \in [-L,L]$ into $j_{\max} +1$ intervals using grid points that are $\Delta = \frac{2L}{j_{\max}}$ distance apart, so that each grid point is of the form  $z_j = -L + j \Delta$ for a natural number $j \in \nat$. Here, $j_{\max}$ is a user-specified input, fixed once $L$ and the discretization factor $\Delta$ are chosen (equivalently, $j_{\max} = \nicefrac{2L}{\Delta}$): a larger $j_{\max}$ (smaller $\Delta$) yields a finer grid and hence a smaller discretization error, at the cost of a longer running time for the dynamic program. We evaluate the function at each of these grid points and approximate the function over this finite set of discrete grid points. Later, we will show how to account for the error introduced by the interval between the grid points that is not accounted for in this process. 
The discretized optimal PWA problem is stated as follows:

\begin{problem}[Optimal Univariate PWA Approximation]\label{prob:univariate-pwa}
The problem of optimally approximating  $\psi(z)$ is as follows:
\begin{compactdesc}
\item[Inputs:] An algorithm for $\psi$, limits $-L, L$, discretization factor $\Delta$ and limit on number of pieces $k$.
\item[Output:] Breakpoints $z_1 \leq z_2 \leq \cdots \leq z_{k-1}$, which together with the values $y_i = \psi(z_i)$ define (via Def.~\ref{def:cpwa-function}) a $k$-piece PWA function $\psihat$, such that the discretized error $e(\Delta) = \max_{u \in \{ -L, -L + \Delta, -L + 2 \Delta, \cdots, L\}} \left| \psi(u) - \psihat(u) \right| $ is minimized.
\end{compactdesc}
Note that if two breakpoints $z_j, z_{j+1}$ coincide in the output, we treat them as the single breakpoint.
\end{problem}

Let $\nn(j; \Delta) = -L + j \Delta$. We write $\nn(j) := \nn(j; \Delta)$ whenever $\Delta$ is clear from context, so $\nn(j)$ and $\nn(j; \Delta)$ denote the same grid point.
The optimal univariate PWA approximation can be solved using a dynamic programming algorithm first proposed by Bellman et al.~\cite{Bellman/1961/Approximation,Bellman+Roth/1969/Curve}. The algorithm defines a value function $V(p, j, q)$. Intuitively, $V(p, j, q)$ is the optimal (minimum) worst-case error for approximating $\psi$ over the range $[\nn(p), \nn(j_{\max})]$ using at most $q$ pieces. The first of these pieces is anchored at the left endpoint $\nn(p)$, and the index $j$ ranges over the candidate right endpoints for this first piece that the recursion considers, with $p < j \leq j_{\max}$ and $q \leq k$, with all points in $[\nn(p), \nn(j)]$ treated as belonging to this same first piece. The function $V$ is recursively defined as
\begin{equation}\label{eq:value-function}
	V(p, j, q) = \begin{cases}
	\mathsf{sPE}(\psi, p, j_{\max}), &\text{if } j \geq j_{\max}\\
		\infty, & \text{if }  q \leq 0                                    \\
        \min(\max(\mathsf{sPE}(\psi, p, j),\\\hspace{4.2em}V(j, j+1, q-1)),\\\hspace{2em}V(p, j+1, q)),&\text{otherwise}
	\end{cases}
\end{equation}

The function $\mathsf{sPE} = \mathsf{singlePieceError}(\psi, j_1, j_2)$ measures the error between $\psi(z)$ and the line joining the points $(\nn(j_1), \psi(\nn(j_1)))$ and $(\nn(j_2), \psi(\nn(j_2)))$, over the discretized interval $[\nn(j_1), \nn(j_2)]$. Let $s_{j_1, j_2}$ be the slope of the line which equals $\frac{ \psi(\nn(j_2)) - \psi(\nn(j_1))}{\nn(j_2) - \nn(j_1)}$
\begin{multline*}
    \mathsf{sPE} =  \max_{z \in \{ \nn(j_1),  \ldots, \nn(j_2)\}} \ \left| \psi(z) - \psihat(z)\right| \,,\\ 
     \text{wherein} \ \psihat(z) = \left( \psi(\nn(j_1)) + (z - \nn(j_1)) s_{j_1, j_2} \right) 
\end{multline*}
The dynamic programming formulation upon memoization runs in time $O(j_{\max}^3 k)$ assuming that evaluations of $\psi(z)$ and other arithmetic operations are $O(1)$. The optimal error equals the value $V(0, 1, k)$ and the solution can be recovered in $O(j_{\max}^2 k)$ steps by traversing the table.
Let $\psihat$ be the function obtained as a result of running Bellman's algorithm. 

\paragraph{Discretization vs. Continuous Error}  We consider the gap between the discretized error $e(\Delta) = \max_{z \in \{ -L, -L + \Delta, -L + 2 \Delta, \cdots, L\}} \left| \psi(z) - \psihat(z) \right| $ and error over the entire continuum of values: $e(\psihat, \psi) = \max_{z \in [-L, L]} | \psi(z) - \psihat(z) |$. We will derive a correction term that can be used to account for this gap. This correction can be used to modify the $\mathsf{singlePieceError}$ calculation used in Eq.~\eqref{eq:value-function}.

The derivation of this error is detailed  in
\ifextendedversion Appendix~\ref{app:discretization-correction-proof}
\else
the extended version~\cite{Schwartz+Others/2026/Optimized}
\fi 
and uses the \textsc{DerivativeInterval} procedure that is assumed to exist through Assumption~\ref{assum:psi}. Finally, the existence of this bound  provides us a way to control  $\Delta$ to a ``small enough'' value so that the discretization error can be made to fall below a user-provided limit.

\paragraph{Trade-Off Table} Finally, we use the dynamic programming formulation to yield a tradeoff table that has the following useful information for a given function $\psi(z)$. Let $k_{\max}$ be an upper limit on the number of breakpoints. The tradeoff table stores, for each value $k \in [1, k_{\max}]$, the minimum error $\err(\psi, k)$ obtained by running Bellman's algorithm seeking a function $\psihat_k$ with at most  $k$ pieces and information on the optimal breakpoints. Note that we do not really need to run the dynamic programming $k_{\max}$ times, once for each budget $k \in [1, k_{\max}]$. The recursion in Eq.~\eqref{eq:value-function} only ever decreases the piece-budget argument $q$, by exactly one each time a piece is closed off. Consequently, computing $V(0, 1, k_{\max})$ by memoization necessarily computes and caches $V(0, 1, q)$ for every intermediate budget $q \leq k_{\max}$ along the way. So a single run of Bellman's algorithm with $k = k_{\max}$ already populates the full memo-table, and we can simply read off the value $V(0, 1, k)$ for every $k \in [1, k_{\max}]$ from it, without any additional runs.

\section{Optimized PWA for Entire Network }\label{sec:optimal-pwa-for-entire-kan}

We will now consider how to approximate the entire network into a PWA function, while ensuring that the overall error $\max_{\vx} | f_N(\vx) - \fhat_N(\vx) | \leq \delta$ for a fixed $\delta$. 
Let us suppose we replace each unit in the KAN $\psi_{i,j}$ by a PWA function $\psihat_{i,j}$ incurring an error $\err_{i,j} = \max_{z \in [-L, L]} | \psi(z) - \psihat(z)|$. We will derive a bound for the entire NN for the value $\max_{\vx \in X} | f_N(\vx) - \fhat_N(\vx)|$.
First, we prove an error propagation bound for each unit $\psi_{i,j}$ under inputs $z, \hat{z}$.
\begin{lemma}\label{Lemma:lipschitz-unit}
    For each unit $\psi_{i,j}$, there exists a constant $ \lip_{i,j} \geq 0$ such that for all $z, \hat{z} \in [-L_{i,j}, L_{i,j}]$,
    \begin{equation*}
        |\psi_{i,j}(z) - \psihat_{i,j}(\hat{z}) |  \leq \err_{i,j} + \lip_{i,j} | z - \hat{z}|
    \end{equation*}
\end{lemma}
\begin{proof}
        Let $\psi=\psi_{i,j}$, $\psihat = \psihat_{i,j}$, $\psi'_{\max} = \max_{z \in [-L_{i,j}, L_{i,j}]} \left|\frac{d\psi}{dz}\right|$ and $\err= \err_{i,j}$. We may write
	\begin{equation*}
	    | \psi(z) - \psihat(\hat{z}) | \leq  |  \psi(z)  - \psi(\hat{z}) |  + |  \psi(\hat{z}) -  \psihat(\hat{z}) | \leq  \psi'_{\max} | z - \hat{z} | + \err
	\end{equation*}
	We obtain the result by setting   $\lip_{i,j} =\psi'_{\max}$. Note that the \textsc{DerivativeInterval} (Assumption~\ref{assum:psi}) function can be used to compute a tight upper bound on $\psi'_{\max}$.
\end{proof}
Consider a NN as in Def.~\ref{def:general-nn} with $K$ layers, input $\vx \in X $ and output $y$.  Let $W_1, \ldots, W_{K}$ be the matrices for the hidden layers and $W_{K+1}$ be the output layer matrix. Also, let $\psi_{i,j}$, the $j^{th}$ activation function at the $i^{th}$ layer be approximated by $\psihat_{i,j}$ with the local approximation error satisfying Lemma~\ref{Lemma:lipschitz-unit}. 
Consider the PWA approximated network obtained by replacing each $\psi_{i,j}$ by $\psihat_{i,j}$.
Let $ \vz_1, \ldots, \vz_{K+1}$ be the layer inputs, with $\vz_1 = \vx$ and $\vz_{i+1} = \sigma_i( W_i \vz_i)$ for $i = 1, \ldots, K$. Similarly, for the PWA approximated network, let $\hatz_i$ represent the corresponding input to the $i^{th}$ hidden layer in the PWA approximated network. Let $\lip_i$ be the vector whose entries correspond to $\lip_{i,j}$ for the activation function $\psi_{i,j}$ from Lemma~\ref{Lemma:lipschitz-unit} and $\ve_{i}$ be the vector whose entries correspond to the approximation error $e_{i,j}$ for the activation function $\psi_{i,j}$.  Let $|W_i|$ represents the matrix whose entries are the absolute values of the corresponding entries in $W_i$ and $|\vz_i - \hatz_i| $ be the vector whose entries consist of the  absolute value of entries in $\vz_i - \hatz_i$.

\begin{lemma}\label{Lemma:layer-approximation}
For $1 \leq i \leq K$, we have $|\vz_{i+1} - \hatz_{i+1} | \leq \lip_i \odot (|W_i| | \vz_i - \hatz_i|) + \ve_i$.
\end{lemma}
\begin{proof}
First, we observe that $|W_i \vz_i - W_i \hatz_i | \leq |W_i| | \vz_i - \hatz_i| $. Next, applying Lemma~\ref{Lemma:lipschitz-unit}, we obtain 
\[
|\vz_{i+1} - \hatz_{i+1} | \leq \lip_i \odot (|W_i| | \vz_i - \hatz_i|) + \ve_i
\]
 Recall that for two vectors $\va, \vb$ of the same dimension,  we define $\va \odot \vb$ as the vector whose entries are products of the corresponding entries $a_i b_i$.
\end{proof}

We will now state the main result that casts the overall error of each output of the NN in terms of the approximation error $\err_{i,j}$ at each unit.   $ \vy = f_N(\vx)$ and  $\vec{\hat{y}} = \fhat_N(\vx)$ for some $\vx \in X$. Let $y_l$ and $\hat{y}_l$ be the $l^{th}$ component of the output for $1 \leq l \leq m$.

\begin{restatable}{theorem}{kanoverallerr}
\label{thm:kan-overall-error}
       For each  unit $\psi_{i,j}$ and for each output index $1 \leq l \leq m$ there exists a constant $\decor{S}{l}{i,j} \geq 0$ such that the overall approximation error is
       $ | y_l - \yhat_l | \leq  \sum_{i=1}^K \sum_{j=1}^{n_{i}}  \decor{S}{l}{i,j} \times \err_{i,j}$.
       Here $\decor{S}{l}{i,j}$ denotes sensitivity weight constants (distinct from the layer weight matrices $W_1,\ldots,W_{K+1}$ of Def.~\ref{def:general-nn} and the knapsack budget $W$ of Problem~\ref{MOK-Problem}).
\end{restatable}
\begin{proof}
Note that $\vz_1 = \hatz_1 = \vx$ (both networks receive the same input). Therefore, $|\vz_1 - \hatz_1| = 0$. From Lemma~\ref{Lemma:layer-approximation}, we note that $|\vz_{i+1} - \hatz_{i+1}| \leq \lip_i \odot (|W_i| | \vz_i - \hatz_i|) + \ve_i$. Expanding on this, we note that for each entry $|z_{i+1, r} - \hat{z}_{i+1,r}|$ the bound takes the form
\[ |z_{i+1, r} - \hat{z}_{i+1,r}| \leq \sum_{k=1}^i \sum_{j=1}^{n_k} \decor{S}{r}{k,j} e_{k,j} \]
for some positive constants $\decor{S}{r}{k,j}$ that bound the contribution of the error incurred in the PWA approximation of $\psi_{k,j}$ on the $r^{th}$ input to the $(i+1)^{th}$ hidden layer.

Finally, the output $y_l$ is obtained as the $l^{th}$ component of the vector $\vy = W_{K+1} \vz_{K+1}$. Therefore, we obtain the required bound as a weighted linear combination of approximation error for each activation function.
\end{proof}

This bound is a worst-case surrogate and is not guaranteed to be tight. Most of the slack we see comes from Lemma~\ref{Lemma:lipschitz-unit} which uses the Lipschitz constant $\Lambda_{i,j} = \max_{z \in [-L_{i,j}, L_{i,j}]} |\psi'_{i,j}(z)|$ over each unit's entire input range rather than the local slope at the inputs actually encountered, and Lemma~\ref{Lemma:layer-approximation} bounds $|W_i \vz_i - W_i \hatz_i|$ by $|W_i|\,|\vz_i - \hatz_i|$, which throws away any cancellation between positive and negative entries, essentially the same sign-agnostic propagation used in interval arithmetic (Section~\ref{sec:related}). The bound is still useful for the knapsack of Section~\ref{sec:optimalkan-abs}, which only needs the \emph{relative} sensitivity $\decor{S}{l}{i,j}$ across units to be roughly right rather than the absolute error estimate to be tight. Section~\ref{sec:results} suggests that this is sufficient in practice.

\subsection{Optimized NN Abstraction as Knapsack}\label{sec:optimalkan-abs}

Theorem~\ref{thm:kan-overall-error} shows how to calculate the error once we have decided upon a PWA abstraction of each unit $\psi_{i,j}$ yielding the error term $\err_{i,j}$. The constants $\decor{S}{l}{i,j}$ corresponding to each unit and output can be computed using a layer-by-layer error propagation analysis that was presented in the proof of the theorem.  For the remainder of this section, we will assume that the network has a single output, and we write $S_{i,j}$ for the corresponding sensitivity weight $\decor{S}{1}{i,j}$. Alternatively, we may specify a weight vector that combines error bound for each component of the output into an overall weighted error bound to be minimized.

The key problem is therefore to decide how many pieces $p_{i,j}$ to allocate to each unit $\psi_{i,j}$, given a fixed global budget $B$ on the total number of pieces used across the entire network, i.e., $\sum_{i,j} p_{i,j} \leq B$, so as to minimize the resulting overall approximation error $\sum_{i,j} S_{i,j} \times \err_{i,j}(p_{i,j})$. We fix a budget on the total number of pieces, rather than on the error, because it is this number that determines the complexity of the downstream verification problem, especially the MILP formulation that we will provide in Section~\ref{sec:milp-formulation}: for a fixed complexity budget $B$, we then obtain the allocation of pieces to units that makes the best use of it, i.e., that minimizes the resulting approximation error.
We formulate the problem of optimized abstraction as a variant of the well-known knapsack problem called the multi-choice knapsack problem~\cite{Kellerer+Others/2004/MultipleChoice}.

\begin{problem}[Multi-Choice Knapsack Problem]\label{MOK-Problem}
An instance of the multi-choice knapsack problem is defined as follows:
\begin{compactdesc}
\item[Inputs:] We have $N \geq 1$ \emph{options}, wherein each option is a set of  $k \geq 1$ items that are available to be chosen and the items belonging to two different options are disjoint. Furthermore, for each option $i$, we have a table of non-negative weights (listed in ascending order) and values for the items in option $i$:
    \[\begin{array}{|r|rrrr|}
            \hline
            \text{Item}    & \decor{I}{i}{1}   & \decor{I}{i}{2}   & \cdots & \decor{I}{i}{k}   \\
            \hline
            \text{Weights} & \decor{w}{i}{1} < & \decor{w}{i}{2} < & \cdots & < \decor{w}{i}{k} \\
            \text{Value}   & \decor{v}{i}{1}   & \decor{v}{i}{2}   & \cdots & \decor{v}{i}{k}   \\
            \hline
        \end{array}    \]
    We are given a total weight budget $W \geq 0$. We will assume \emph{feasibility}: i.e, $\sum_{i=1}^N \decor{w}{i}{1} \leq W$.
\item[Output:] Choose a set of $N$ items, choosing precisely one item from each option such that (a) the total weight of chosen items remains within the budget $W$ and (b) the total value of the chosen items is maximized.
\end{compactdesc}
\end{problem}

\begin{restatable}{theorem}{moknapsacknpcomplete}\label{theorem:moknapsack-npcomplete}
    The (decision version) of the multi-choice knapsack problem is \textbf{NP}-complete.
\end{restatable}
Proof is provided in \ifextendedversion
Appendix~\ref{appendix:proof-theorem3}.
\else
the extended version~\cite{Schwartz+Others/2026/Optimized}.
\fi
Choosing an optimized abstraction for a NN reduces to a multi-choice knapsack problem:
\begin{inparaenum}[(a)]
        \item Each option corresponds to a single unit $\psi_{i,j}$ and thus, there are as many options as the number of units in the KAN;
        \item For each option corresponding to the unit $\psi_{i,j}$, each item corresponds to a choice of the number of pieces $p \in [1, k_{\max}]$ for the piecewise linearization $\psihat_{i,j}$, with weight $\decor{w}{i,j}{p} = p$, the number of pieces itself, and value $\decor{v}{i,j}{p} = -S_{i,j} \times \err_{i,j}(p)$, wherein $\err_{i,j}(p)$ is the minimum error achievable using $p$ pieces for unit $\psi_{i,j}$, obtained from the trade-off table discussed in Section~\ref{sec:optimal-pwa-univariate-fun}. Using more pieces for a unit lowers its error $\err_{i,j}(p)$ and hence raises its value $\decor{v}{i,j}{p}$, at the cost of a larger weight.
        \item The weight budget $W$ of Problem~\ref{MOK-Problem} is the same as the total piece budget $B$.
\end{inparaenum}
Maximizing the total value $\sum_{i,j} \decor{v}{i,j}{p_{i,j}}$ subject to the weight budget $B$ is therefore the same problem as minimizing the total weighted error $\sum_{i,j} S_{i,j} \times \err_{i,j}(p_{i,j})$ subject to the piece budget $B$, as desired.

Let us assume a multi-choice knapsack problem instance with $N$ options, each having $k$ items. The dynamic programming formulation is based on a value function
that we will denote $\MO(i, \hat{W})$ that asks for the optimal (maximum) total value obtainable by selecting items from options $i$ until $N$ (inclusive) with a remaining weight budget $\hat{W}$.
\begin{equation*}
    \MO(i, \hat{W}) = \begin{cases}
		-\infty\;\; \text{if}\ \hat{W} < 0                                                                                   \\
		0 \;\; \text{if}\  i > N \ \land\ \hat{W} \geq 0                                                                    \\
		\max_{j=1}^k \left( \decor{v}{i}{j} + \MO\left(i+1, \hat{W} - \decor{w}{i}{j} \right) \right) \;\; \\
	\end{cases}
\end{equation*}
The optimal answer is obtained by computing $\MO(1, W)$ through memoization.
\begin{theorem}
	The  dynamic programming algorithm for multi-choice knapsack runs in time $O(W N k)$.
\end{theorem}
In the setting of optimized NN abstractions, the weights $\decor{w}{i,j}{p} = p$ are natural numbers by construction, and the weight budget $W$ is instantiated as the total piece budget $B$. Unlike an error-valued budget, this requires no additional quantization assumption. The running time complexity is linear in the network size, since $N$ is the number of units in the network.

\section{Mixed Integer Linear Program (MILP) Encoding}\label{sec:milp-formulation}

In this section, we will briefly discuss how to solve the range verification problem given a PWA approximation of each activation function $\psi_{i,j}$ by $\psihat_{i,j}$ and corresponding error $e_{i,j}$. Our approach is to note that each PWA function $\psihat_{i,j}$ can be written as a finite weighted sum of ReLUs over the domain $[-L_{i,j}, L_{i,j}]$. Therefore, any approach that can encode ReLU NNs into a MILP can be directly adapted to encode the PWA abstraction for the purposes of solving the range verification problem. Other verification approaches  can be used as well after modifying them to account for the approximation error at each unit.  Recall that $\relu(x) = \max(x, 0)$.

\begin{theorem}\label{thm:relu-encoding}
Any continuous $k$-piece PWA function $\psi(z)$ (according to Def.~\ref{def:cpwa-function}) with $k \geq 2$ pieces and $k-1$ breakpoints can be written as a linear combination of ReLU functions: 
\begin{equation}\label{eq:relu-equivalence}
   \psi(z) = \left( \begin{array}{c}
y_1 - m_0 \relu(z_1 - z) +  m_1 \relu(z - z_1) + \\
\sum_{j=2}^{k-1} (m_j - m_{j-1}) \relu(z -z_j) 
\end{array}\right) 
\end{equation}
\end{theorem}
\begin{proof}
Let $g(z) $ represent the RHS of Eq.~\eqref{eq:relu-equivalence}. We verify that $\psi(z) = g(z)$ by running through the cases $z \leq z_1$ and then each of the cases $ z \in [z_j, z_{j+1}]$ for $j \in [1, \ldots, k-2]$ and $z \geq z_{k-1}$.

First, if $z \leq z_1$, verify that $\psi(z) = y_1 + m_0( z - z_1) = y_1 - m_0 \relu(z_1 - z)$. Also, $\relu(z - z_j) = 0$ for $z \leq z_1$ since $z \leq z_1 < z_j$.

Suppose $z \in [z_j, z_{j+1}]$, then note that $\relu(z_1 - z) = 0$ and
$\relu(z - z_i) = 0$ for all $i \geq j+1$ and $\relu(z - z_i) = (z -z_i) $ whenever $i \leq j$. Therefore, 
\[ \begin{array}{l}
g(z)  - y_1 \\
\; =  m_1 \relu(z - z_1) + \sum_{i=2}^{j} (m_i - m_{i-1}) \relu(z - z_i) \\ 
\; =  m_1 (z - z_1) +  \sum_{i=2}^{j} (m_i - m_{i-1}) (z - z_i) \\ 
\; =  \left( m_1 + (m_2 - m_1) + \cdots + (m_j - m_{j - 1}) \right) z  \\ 
\;\;\;- \underset{\Delta_2}{\underbrace{( m_1 (z_2 - z_1) + \cdots + m_{j-1}(z_{j} - z_{j-1}) )}} - m_j z_j \\ 
\; =  m_j z  - m_j z_j + y_{j} - y_1
\end{array}\]
This is because $m_1 + (m_2 - m_1) + \cdots + (m_j - m_{j-1}) = m_j$ and 
furthermore, $m_1 (z_2 - z_1) = y_2 - y_1$, $m_2 (z_3 - z_2) = y_3 - y_2 $, $\cdots$, $m_{j-1}(z_{j} - z_{j-1}) = y_{j} - y_{j-1}$. Therefore, $\Delta_2 = y_2 - y_1 + y_3 - y_2 + \cdots + y_j - y_{j-1} = y_j - y_1$. Combining, 
$g(z) = y_j + m_j (z - z_j) = \psi(z)$ whenever $z \in [z_j, z_{j+1}]$. An identical argument applies when $z \geq z_{k-1}$ to yield $g(z) = y_{k-1}+ m_{k-1} (z - z_{k-1}) = \psi(z)$. The required identity is thus verified to hold. 
\end{proof}

As a result, the MILP encoding is identical to existing approaches to neural network verification using mixed integer programming~\cite{tjeng_xiao_tedrake_2017,dutta_chen_jha_sriram_sankaranarayanan_tiwari_2019,lomuscio_maganti_2017} with the addition of an extra continuous variable $\epsilon_{i,j}$ at each node to represent the error between the PWA approximation and the original result.  
Due to space limitations, the details of the ReLU encoding are given in \ifextendedversion Appendix~\ref{sec:milp-encoding} \else the extended version~\cite{Schwartz+Others/2026/Optimized}\fi.

Theorem~\ref{thm:relu-encoding} also points to a more general approach to verification. Because every PWA abstraction can be rewritten as a sum of ReLU units, our abstraction is exactly a ReLU network augmented with a bounded, nondeterministic error term $\epsilon_{i,j}$ at each unit. Any verification backend that already handles ReLU networks and admits a bounded slack variable per unit can therefore reason over our abstraction directly. MILP-based approaches handle this natively, as detailed in \ifextendedversion Appendix~\ref{sec:milp-encoding} \else the extended version\fi . SMT-based tools such as Reluplex~\cite{katz_barrett_dill_julian_kochenderfer_2017} and Marabou~\cite{katz_huang_2019} can do so as well, by treating each $\epsilon_{i,j}$ as an additional bounded real-valued variable. Bound-propagation approaches such as CROWN~\cite{zhang_weng_chen_hsieh_daniel_2018} and DeepPoly~\cite{Singh+Others/2019/Beyond} require a further modification to propagate these per-unit error intervals alongside the usual activation bounds. This is not conceptually difficult, but it does require implementation changes rather than a direct as-is application of these tools.

\section{Results}\label{sec:results}
Our framework applies to any activation function satisfying Assumption~\ref{assum:psi}, which covers B-splines (and hence deep-spline networks), Pad\'e approximations, Chebyshev polynomials, Gaussian processes, and Fourier-series activations, in addition to fixed activations such as tanh, sigmoid, and gelu. We perform our empirical evaluation on KAN networks~\cite{li2024kolmogorovarnold} specifically because each unit in a KAN carries its own independently learned activation function, unlike a standard feedforward network that reuses a single fixed nonlinearity (e.g., tanh or gelu) at every node. KANs use learnable representations at each node, meaning each activation function is  non-linear. Crucially, the functions differ across units and  thus a single abstraction for the entire network is not feasible.  This makes KANs a natural test for our method: every unit demands its own abstraction, and the choice of how to approximate each unit becomes a core decision problem.

\begin{table}[t]
\centering
\caption{Benchmark suite: KAN architectures and trainable parameter counts.
Architecture is given as $[n, m_1, \ldots, m]$ where $n$ is the input
dimension, $m_1,\ldots$ are hidden layer sizes, and $m$ is the output
dimension.}
\label{tab:benchmarks_wide}
\begin{tabular}{@{}lcc@{\hskip 4pt}|@{\hskip 4pt}lcc@{}}
\toprule
\textbf{ID} & \textbf{Arch.} & \textbf{Params} & \textbf{ID} & \textbf{Arch.} & \textbf{Params} \\
\midrule
fbessel      & [1, 1]     & 20    & fexp4      & [4, 4, 2, 1]  & 260   \\
fexp         & [2, 1, 1]     & 30    & pinnheat   & [2, 5, 1]     & 420   \\
fxy          & [2, 2, 1]     & 60    & fexp100    & [100, 1, 1]   & 1010  \\
fellipkinc   & [2, 2, 1, 1]  & 70    & fnoise     & [5, 12, 1]    & 1312  \\
fellipeinc   & [2, 2, 1, 1]  & 70    & prosthetic & [50, 8, 4, 5] & 3.6K  \\
flegendre    & [1, 5, 1]     & 100   & weather    & [168, 16, 6]  & 22.2K \\
fsphharm     & [2, 3, 2, 1]  & 140   & acopf\_14  & [22, 32, 186] & 53.2K \\
\bottomrule
\end{tabular}
\end{table}

The KANs used in our analysis are trained using the FastKAN library~\cite{li2024kolmogorovarnold}.
Specifically, FastKAN replaces the B-spline basis functions of the original KAN with radial basis functions (RBFs), reducing computation time with negligible accuracy loss. Each activation function is a learned weighted sum of Gaussian RBFs, with the weights optimized independently per unit during training so each unit in the network carries a distinct nonlinear activation function. However, these functions satisfy Assumption~\ref{assum:psi}. {In the implementation, each spline is given in terms of the value at a densely sampled grid of points, $1000$ uniformly spaced grid points. We did not attempt to implement the correction for the error between the discretized spline and the underlying continuous function alluded to in Section~\ref{sec:optimal-pwa-univariate-fun}, since it was negligible for the cases that we examined manually given the dense set of samples.} 
We use an initial interval bound propagation to fix the domain of each activation function. 

\begin{figure*}[ht!]
    \centering
    \includegraphics[width=1\linewidth, trim=6 6 6 6, clip]{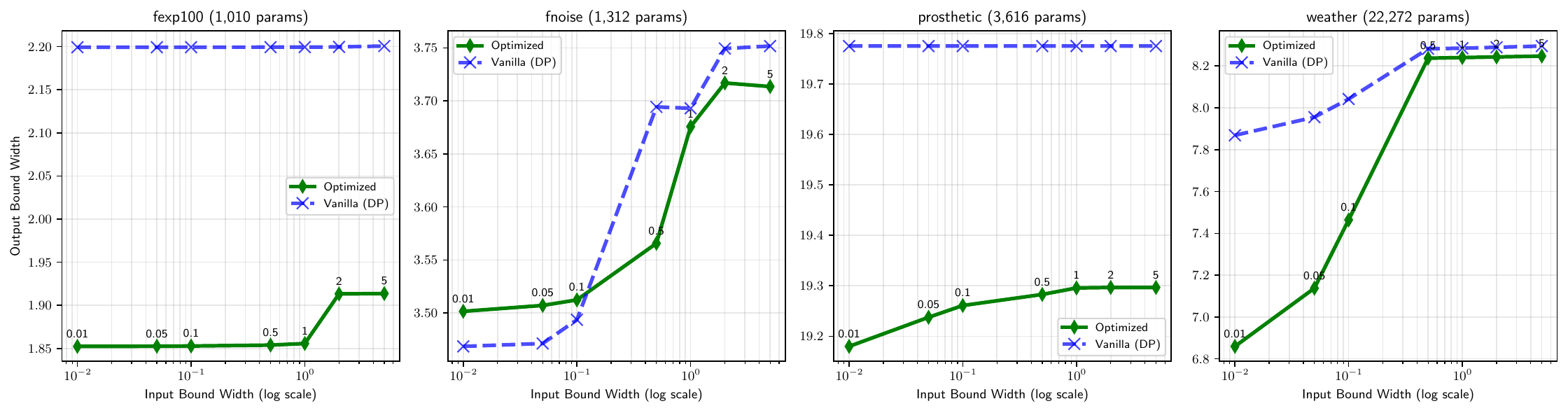}
    \caption{Output bound width as a function of input perturbation radius $r$ over the four largest benchmarks. \emph{Optimized} consistently achieves tighter bounds (smaller width is considered better) compared to the \emph{Vanilla} across all input widths. 
    }\label{fig:iw_width}
\end{figure*}

\begin{figure*}[ht!]
    \centering
    \includegraphics[width=1\linewidth, trim=6 6 6 6, clip]{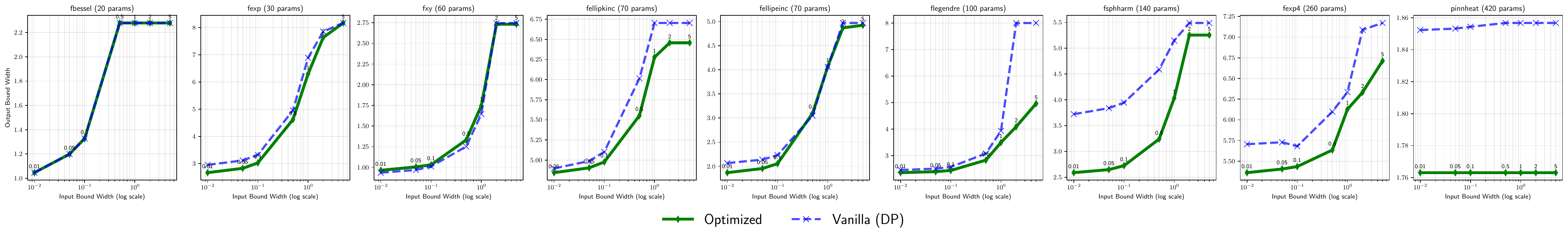}
    \caption{Output bound width as a function of input perturbation radius $r$  for the smaller (additional) benchmarks arranged left to right according to increased parameter sizes. Solid green line reports the output bound widths of the \textit{Optimized} method (smaller is better) while dotted blue line reports output bound widths for \textit{Vanilla} method. 
    }\label{fig:input_width_all}
\end{figure*}

The benchmarks are taken from a variety of sources including the KANs used for function approximations from Liu et al. (\texttt{fbessel}, \texttt{fexp100},..., \texttt{fellipkinc})~\cite{Liu+Others/2025/KAN}, KANs trained on existing datasets (\texttt{prosthetic}, \texttt{weather}, \texttt{acopf\_14}) and KAN model that solves differential equations such as the heat equation (\texttt{pinnheat}). A detailed description is available \ifextendedversion in Appendix~\ref{app:more-benchmarks}\else as part of the extended version of this paper\fi.  



We consider two approaches: (a) \textbf{Vanilla}, the baseline method that uses a fixed number of segments $k$ for each unit and uses Bellman's algorithm to compute PWA at each node; and (b) \textbf{Optimized}, which solves a knapsack problem for total segment budget $B = k \cdot |\mathcal{S}|$ across splines by solving an ILP, where $|\mathcal{S}|$ is the total number of splines in the network. Both approaches use the same total number of pieces: uniformly for \textit{Vanilla}, and via the piece-budget-constrained ILP of Section~\ref{sec:optimalkan-abs} for \textit{Optimized}. Besides these methods, we also compare against the bounds discovered by running an interval bound propagation algorithm (IBP) in Tab.~\ref{tab:ibp-combined}. The \textit{Optimized} pipeline executes four stages for each verification: 
\begin{compactenum}
\item For each activation function $\psi_{i,j}$, we run Bellman's DP to obtain a trade-off curve relating piece count to local error $e_{i,j}$. The DP table at each node is computed for the number of segments ranging from $1$ until $\max(S_{\max}, 2k)$, where $S_{\max} = 16$ is the default maximum; the $2k$ term ensures the table extends beyond the uniform budget when $k$ is large. The \textit{Vanilla} method, on the other hand, simply uses a single fixed $k$ for all splines. 
\item We estimate each activation function's Lipschitz constant to obtain  the sensitivity weights following the approach outlined in the proof of Theorem~\ref{thm:kan-overall-error}. 
\item We encode the multi-choice knapsack as an $0-1$ integer program that finds an optimal allocation of pieces summing up to the total number of pieces to be at most $k \cdot |\cal{S}| $, while minimizing the worst case error estimate.

\end{compactenum}  

In our evaluation, we consider input ranges that are $\ell_\infty$-balls given by $\mathcal{B}_\infty( r) = \{\mathbf{x} \mid \|\mathbf{x} \|_\infty \leq r\}$ with radii $r \in \{0.01, 0.05, 0.1, 0.5, 1.0, 2.0, 5.0\}$ and we sweep over segment counts $k \in \{2, \ldots, 7\}$.  We compare the output bound widths (upper bound minus lower bound) and when the network has multiple outputs, we compute the largest width over the first $5$ outputs as a comparison measure. Running times are measured in seconds: we separate the ``overhead time" of computing the optimal allocation and the time for solving the MILP. The rationale for doing so is that the DP is constructed once for the entire network and amortized across multiple verification tasks.  Our implementation is based on Python 3.13 and uses {GPU-accelerated} Gurobi as the MILP and ILP solver~\cite{gurobi}. All experiments are run with a fixed input domain centered at the origin and with fixed random seeds for reproducibility. We report a single run per benchmark rather than averaging over seeds, since the only randomness is in the initial network training and both the DP and ILP allocation steps are deterministic given a trained network. {We run all the experiments on Nvidia GB10 Blackwell GPUs consisting of 128 GB unified memory}. The timings reported here may vary  from run to run.

We note that the \texttt{acopf\_14} benchmark timed out (timeout = 1000 seconds per MILP instance) for all but the smallest segment budgets for both the Vanilla and Optimized methods. We do not include this benchmark in our analysis but note that it is too large for our current MILP solvers to handle within the designated timeout.

\textbf{Overarching questions:}
We now describe the evaluation method on the benchmarks and ask the following two overarching questions: \emph{(1) Does the Optimized approach yield tighter output bounds than the Vanilla baseline at equal segment budgets? What happens as the segment budget increases?} \emph{(2) Does our approach scale as model parameters grow?}

\paragraph{\textbf{Q1: Bound Tightness --- Optimized vs.\ Uniform Allocation}}

\begin{figure*}[ht!]
    \centering
    \includegraphics[width=1\linewidth, trim=6 6 6 6, clip]{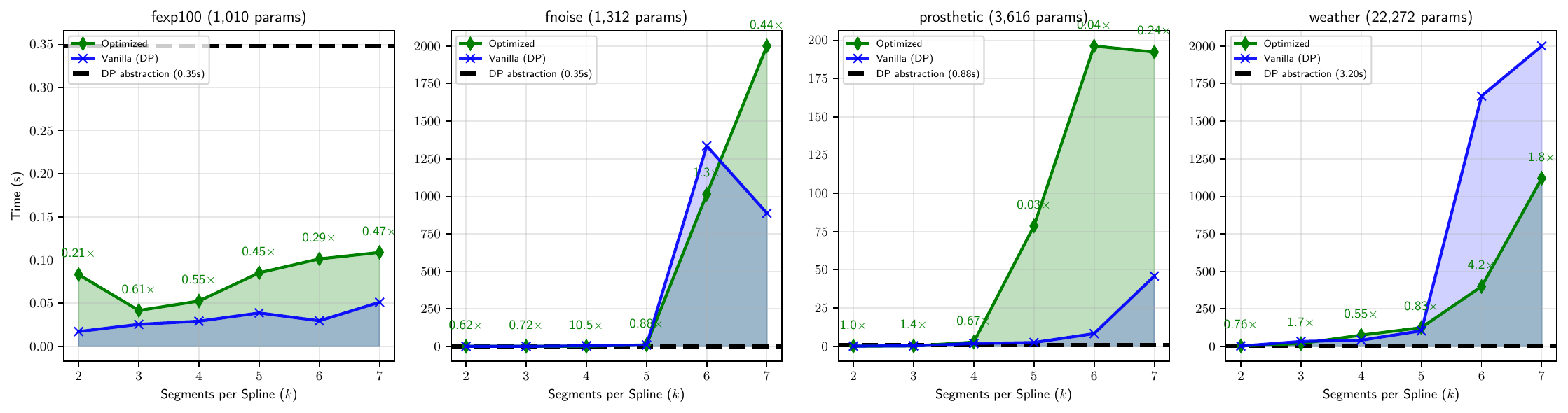}
        \caption{End-to-end runtime as a function of k for the four largest benchmarks. \textit{Optimized}'s per-query cost is the sum of the one-time DP abstraction (black dotted line) and the downstream MILP (green); \textit{Vanilla} incurs only the MILP cost (blue). The multipliers on the green curve show the ratio of \textit{Vanilla} MILP-stage runtime/\textit{Optimized} MILP-stage runtime. Ratio $> 1$ indicates speedup and $<1$ indicates slowdown.}\label{fig:timing}
\end{figure*}

\begin{figure*}[ht!]
    \centering
    \includegraphics[width=1\linewidth, trim=6 6 6 6, clip]{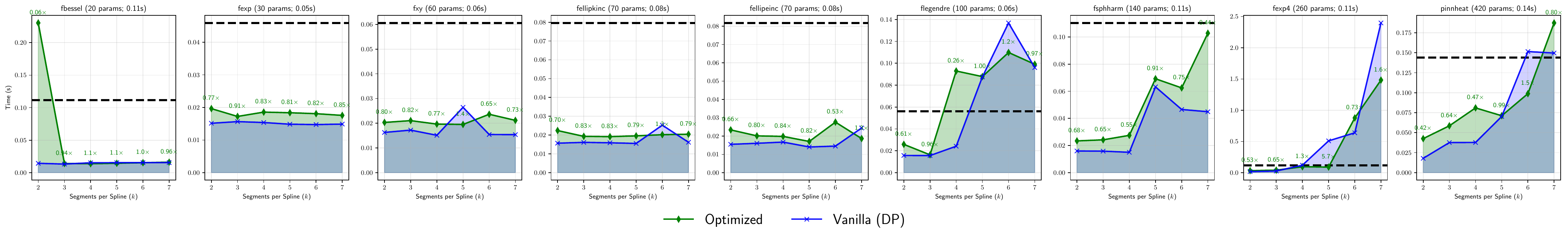}
    \caption{End-to-end runtime as a function of $k$ for the additional smaller benchmarks arranged left to right according to increased parameter sizes. \textit{Optimized}'s per-query cost is the sum of the one-time DP abstraction (black dotted line, with value in the title) and the downstream MILP (green); \textit{Vanilla} incurs only the MILP cost (blue).}\label{fig:timing_small}
\end{figure*}

By varying the input perturbation radius $r$, while  fixing the segment budget $k=5$, we can evaluate the ablative tightness of the proposed method.  Figure~\ref{fig:iw_width} evaluates robustness under varying input perturbation radii over four of the largest benchmarks while Figure~\ref{fig:input_width_all} shows this over the remaining $9$ smaller benchmarks. The comparisons demonstrate that the \textit{Optimized} approach consistently obtains bounds with smaller widths than the \textit{Vanilla}. Also, while the gains vary considerably, they are more pronounced for the larger benchmarks than for the smaller ones. Furthermore, the gains are also more prominent over smaller input widths than larger ones, with a few exceptions including the \texttt{fnoise} and \texttt{flegendre} benchmarks. Note that the parameter size is a poor proxy for the overall complexity of the network, which needs to account for the nature of the splines, and how errors made in the earlier hidden layers propagate to the output. These are network specific attributes that are hard to summarize in a succinct manner. 



 Table~\ref{tab:ibp-combined} compares the output bound widths obtained by the \emph{Optimized} and \emph{Vanilla} approaches against a standard interval bound propagation (IBP) fixing the average number of segments per spline $k=5$ and for two different input bound radii $r = 0.01, 0.1$.  The \textit{Optimized} method outperforms the other approaches in terms of the output bound widths across a vast majority of the benchmark instances considered.

\begin{table}[t]
\centering
\caption{Output bound widths from interval bound propagation (IBP) compared with
\textit{Vanilla} (Van.) and \textit{Optimized} (Opt.) methods, obtained by fixing $k=5$, for two
values of the input width $r$. Benchmarks are sorted by increasing trainable
parameter count. The smallest (best) bound width within each $r$ group is shown
in \textbf{bold}.}
\label{tab:ibp-combined}
\begin{tabular}{lcccccc}
\toprule
& \multicolumn{3}{c}{$r=0.01$} & \multicolumn{3}{c}{$r=0.1$} \\
\cmidrule(lr){2-4} \cmidrule(lr){5-7}
ID & IBP & Van. & Opt.  & IBP & Van. & Opt. \\
\midrule
fbessel    & \textbf{1.046} & \textbf{1.046} & \textbf{1.046} & \textbf{1.327} & \textbf{1.327} & \textbf{1.327} \\
fexp       & 2.958 & 2.958 & \textbf{2.672} & 3.325 & 3.325 & \textbf{3.030} \\
fxy        & 0.937 & \textbf{0.936} & 0.964 & 1.019 & \textbf{1.011} & 1.031 \\
fellipeinc & 2.066 & 2.066 & \textbf{1.596} & 2.230 & 2.230 & \textbf{1.747} \\
fellipkinc & 4.896 & 4.896 & \textbf{4.844} & 5.102 & 5.102 & \textbf{4.976} \\
flegendre  & 2.461 & 2.456 & \textbf{2.045} & 2.639 & 2.571 & \textbf{2.109} \\
fsphharm   & 3.722 & 3.722 & \textbf{2.573} & 3.947 & 3.946 & \textbf{2.691} \\
fexp4      & 6.485 & 5.705 & \textbf{5.373} & 6.563 & 5.680 & \textbf{5.419} \\
pinnheat   & 1.852 & 1.852 & \textbf{1.814} & 1.854 & 1.854 & \textbf{1.814} \\
fexp100    & 0.260 & 0.259 & \textbf{0.235} & 0.283 & 0.282 & \textbf{0.254} \\
fnoise     & 2.081 & 1.994 & \textbf{1.918} & 2.172 & 2.071 & \textbf{1.979} \\
prosthetic & 0.752 & \textbf{0.748} & 0.792 & 1.009 & \textbf{0.989} & 1.018 \\
weather    & 0.715 & 0.715 & \textbf{0.511} & 0.877 & 0.877 & \textbf{0.691} \\
\bottomrule
\end{tabular}
\end{table}

\begin{figure*}[ht!]
    \centering
    \includegraphics[width=1\linewidth, trim=6 6 6 6, clip]{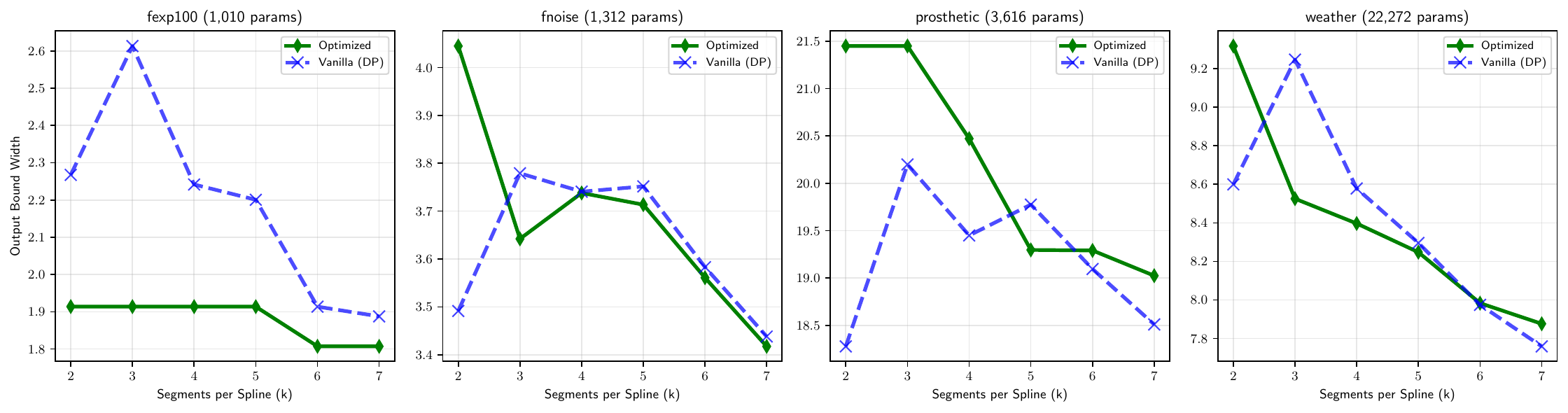}
    \caption{Output bound width as a function of per-spline segment budget $k$ for fixed input bound width $r=5$.  The solid green curve plots the trend for the \textit{Optimized} method while the dashed blue curve shows that for the \textit{Vanilla} method.}\label{fig:seg-vs-output-bound-width}
\end{figure*}

As a related comparison, Figure~\ref{fig:seg-vs-output-bound-width} shows the output bound width as a function of the segment budget $k$ for the four largest benchmarks. Although the \textit{Optimized} method performs better, there are a few instances where the \textit{Vanilla} approach performs better. We also note that while the optimized approach generally shows a tightening of the bound widths as the number of segments increases, the \textit{Vanilla} method does not show this trend in many of the benchmarks. Indeed, a systematic analysis using dynamic programming allocates more segments to splines that have a higher impact on the output bound, which in turn shows a monotone behavior of output bound width versus number of segments. This is not true, in general, for the \emph{Vanilla} method which allocates the same number of segments to all the splines.

\paragraph{\textbf{Q2: Computation Time Comparison}}

Figure~\ref{fig:timing} reports end-to-end verification time as a function of the segment budget for the four largest benchmarks. The dynamic programming time dominates the MILP times for the smallest benchmarks. On the other hand for the largest three benchmarks, it is negligible compared to the time taken by the MILP stages. Also, for applications that involve multiple verification queries over the same network~\cite{SinghEtAl2019AbstractDomain}, it is possible to amortize the one time cost of a DP run against multiple runs of the MILP. 
This is important because a single abstraction of the network can be reused across an unlimited number of verification queries such as varying input radius, target output property, or input feature for sensitivity analysis without re-running the DP.  Also, note that the DP partly parallelizes across splines, since the trade-off table can be computed for each individual spline, in parallel. 

As such, the number of binary variables in the MILP encoding is linear in $ k |\cal{S}|$, and  is the key driver of its worst-case complexity.   Since this is held to be the same across the \textit{Vanilla} and \textit{Optimized} methods, we do not expect to see a marked advantage for either method. However, the actual run times can be variable since MILP solvers involve a large number of heuristics and optimizations that can exploit particular properties of problem instances. In general, this expected behavior is largely borne out in the comparisons shown in Figures~\ref{fig:timing} and~\ref{fig:timing_small}. The key takeaway is that the overhead to compute the optimal abstractions remains within the same order of magnitude for small benchmarks and is completely dominated by the MILP times for the larger ones. 

\paragraph{Discussion} Our experiments on KANs spanning $20$ to $22K$ parameters demonstrate the value of the overall abstraction scheme. On most benchmarks the approach yields tighter output bound widths while remaining comparable in terms of running time. Also, the dynamic programming algorithm used to compute the optimized abstraction can be partly parallelized, yielding speedups and reducing the overhead of the initial abstraction phase considerably. 

Beyond IBP, no existing verification technique can be applied as is to the networks studied in this paper without first applying the abstraction techniques discussed in this paper. Although the work of Shi et al.~\cite{Shi+Others/2025/Neural} is closely related, their approach treats a fixed library of activation functions and fixes a set of breakpoints \emph{a priori}, as mentioned earlier in Section~\ref{sec:related}. Furthermore, although our PWA abstractions connect to many of the existing neural network verification approaches, especially through the translation through ReLU networks enabled by Theorem~\ref{thm:relu-encoding}, we still need to modify some of these approaches to handle the error between the original and PWA activation functions. Though conceptually simple, such a modification will need implementation changes. As a result, further benchmarking of our approach will be performed as part of our future work. 

We also note that the error propagation bound of Theorem~\ref{thm:kan-overall-error} is a worst-case sensitivity bound, and it may not capture the ``average-case'' or ``evaluation-time'' sensitivity of how the local PWA error at a single spline affects the overall output error. Investigating techniques to capture this average-case sensitivity is an important direction for future work. That said, for the knapsack optimization to succeed, it is more important to correctly capture the \emph{relative} weighting of errors across nodes than to obtain a precise absolute estimate. In this sense, our results, where the \emph{Optimized} allocation consistently outperforms uniform \emph{Vanilla} allocation, suggest that the worst-case sensitivity weights are already informative enough to guide the DP-based optimization toward a better abstraction.
\section{Conclusion}
We have presented a general framework for the range verification of neural networks with nonlinear and learnable activation functions by treating the choice of piecewise affine abstraction as an optimization problem. We bound network-wide error as a weighted linear combination of local errors and cast the problem of allocating pieces across units as a multi-choice knapsack problem. A key property of the framework is that the abstraction is computed \emph{once per network} and reused across any number of verification queries. This amortization is precisely what makes the upfront cost of optimized allocation worthwhile in practice, and it matches the structure of real certification pipelines, which issue many queries per model~\cite{SinghEtAl2019AbstractDomain}.

Future work will focus on reducing abstraction overhead for smaller networks, for instance via lazy or incremental DP, {validating the surrogate directly against a sampling-based estimate of the true error on our benchmarks,} and incorporating {property- and} MILP complexity estimates directly into the allocation objective to further improve end-to-end solver performance.

\paragraph*{Acknowledgments}
The authors thank the anonymous reviewers for their comments. This work was supported by the US National Science Foundation under award numbers 2422136 and 2534982. All opinions expressed here are those of the authors and not necessarily of the National Science Foundation. 


\balance
\bibliographystyle{IEEEtran}
\bibliography{bare_conf}

\ifextendedversion
\newpage
\onecolumn
\appendices
\section{Discretization Error}\label{app:discretization-correction-proof}
\begin{restatable}{theorem}{discretizationthm}

\label{thm:discretization-correction}
	Consider any interval $[\nn(j_1), \nn(j_2)]$ for $ 0 \leq j_1 < j_2 \leq j_{\max}$ and let $e_{j_1, j_2}$ denote the error between
	$\psi$ and $\psihat$ over all the discretization points in the interval $[\nn(j_1), \nn(j_2)]$. Also, let $c_{\max}$ be the
	maximum absolute value of the slope of any piece of $\psihat$ over the interval
	\begin{equation*}
	    e_{j_1, j_2} \leq \max_{z \in [\nn(j_1), \nn(j_2)]} | \psi(z) - \psihat(z) | \leq e_{j_1, j_2} + (c_{\max} + \psi'_{\max}) \Delta\,
	\end{equation*}
	wherein $\psi'_{\max} = \max_{z \in [\nn(j_1), \nn(j_2)]} \left| \frac{d \psi}{dz} \right| $.
\end{restatable}
We start by proving a useful lemma that derives an error bound for an interval between two consecutive grid points $[\nn(j), \nn(j+1)]$ for $1 \leq j < j_{\max}$.
\begin{lemma}\label{Lemma:useful-lemma-1}
    Let $e_j = | \psi(\nn(j)) - \psihat(\nn(j)) |$ for some index $j \in [1, j_{\max})$. For any $z \in [\nn(j), \nn(j+1)]$, the error $ | \psi(z) - \psihat(z) | \leq e_j + \left( \psi'_{\max} + |c| \right) \Delta $, wherein $\psi'_{\max} = \max_{z \in [\nn(j), \nn(j+1))} \left| \frac{d \psi}{dz} \right|$ and $ \psihat(z) = c z + d$ for $z \in  [\nn(j), \nn(j+1)]$
\end{lemma}
\begin{proof}
    Consider any $z \in [\nn(j), \nn(j+1)]$. Note that by the mean value theorem: $\psi(z) = \psi(\nn(j)) + \psi'(z^*) (z - \nn(j))$ for some $z^* \in [\nn(j), z]$. Therefore, we have, $ | \psi(z) - \psi(\nn(j))| \leq | \psi'(z^*)| \times  | z - \nn(j)| \leq \psi'_{\max} \Delta$ which leads us to:
    \begin{align*}
        | \psi(z) - \psihat(z) | 
            & \leq | \psi(z) - \psi(\nn(j)) + \psi(\nn(j) - \psihat(\nn(j)) + \psihat(\nn(j)) - \psihat(z) | \\
		& \leq | \psi(z) - \psi(\nn(j)) | + | \psi(\nn(j) - \psihat(\nn(j)) | + |  \psihat(\nn(j)) - \psihat(z) | \\
            & \leq \psi'_{\max} \Delta + e_j + | c (z - \nn(j)) |\;\; \leq \;\; e_j + \left( \psi'_{\max} + |c| \right) \Delta \qedhere
	\end{align*}
\end{proof}
Note that $\psi'_{\max}$ can be computed using the function \textsc{DerivativeInterval} assumed in Assumption~\ref{assum:psi}. Now, we can use Lemma~\ref{Lemma:useful-lemma-1} to derive a correction for $\mathsf{singlePieceError}$ using the following theorem:

\discretizationthm*
\begin{proof}
        Follows directly from applying Lemma~\ref{Lemma:useful-lemma-1} over interval $[\nn(j), \nn(j+1)]$ for $ j_1 \leq j < j_2$. 
\end{proof}

\section{Proof of Theorem~\ref{theorem:moknapsack-npcomplete}} \label{appendix:proof-theorem3}

\moknapsacknpcomplete*
\begin{proof}
    The decision Multi-Choice Knapsack (MCK) problem can be shown to be \textbf{NP}-complete by a reduction from the classical $0-1$ knapsack problem, which is known to be \textbf{NP}-complete (which in turn is reduced from the subset-sum problem).

    Given a candidate solution for the decision version of Problem~\ref{MOK-Problem} we can perform the following checks i.e. \begin{inparaenum}[(a)]
        \item does the total weight of the selected items fall within the budget specified $W$?
        \item does the total value of all items amount to at least the target value $V$? and
        \item is exactly one item selected from each option?
    \end{inparaenum}. We can verify all these conditions in $O(Nk)$ polynomial time, leading us to $\text{MCK} \in \textbf{NP}$.

    Now, we can reduce from the decision version of the $0-1$ knapsack problem, where we can set the problem instance as selecting a set of $N = m$ items each with weight $w_j \in \mathbb{N}$, value $v_j \in \mathbb{N}$, a weight bound $W \in \mathbb{N}$, and a maximum target value $V \in \mathbb{N}$. We need to find out: Is there a subset $S \subseteq \{1, \dots, m\}$ such that $\sum_{j \in S} w_j \leq W$ and $\sum_{j \in S} v_j \geq V$?

    Given such an instance, we can formulate the problem where, for each item $i$ in the $0-1$ knapsack problem, we can define an option $k$ with two items:
    \begin{compactenum}
        \item Item $I^k_{i,1}$ with weight $w^k_{i,1} = 0$ and value $v^k_{i,1} = 0$, corresponding to not selecting the item $j$
        \item Item $I^k_{i,2}$ with weight $w^k_{i,2} = w_j$ and value $v^k_{i,2} = v_j$, corresponding to selecting item $j$
    \end{compactenum}
    If the knapsack problem has a subset $S$ satisfying the constraints of the MCK problem, then for each $i$ choosing item $I^i_{i,2}$ if $i \in S$ and $I^i_{i,1}$ otherwise, will lead to a valid MCK solution with total weight $\leq W$ and value $\geq V$. Conversely, any feasible solution corresponds to a selection of items where each option contributes either no weight and value or the weight and value corresponding to the item selected. In that case, the set of indices $j$ where $I^j_{j,2}$ is selected forms a feasible solution to the $0-1$ knapsack problem.
    
    It is clear that this process of converting the problem is polynomial time in number of inputs $N$. Thus, we have shown that $0-1$ knapsack $\leq_p$ MCK.
\end{proof}

\section{MILP Encoding}\label{sec:milp-encoding}

We will recall the MILP encoding of ReLU networks and adapt it to our context where the PWA functions have a fixed error bound. It is convenient to view the network as a directed acyclic graph with three types of nodes:
\begin{compactenum}
\item Activation function nodes that are PWA functions $\psihat_{i,j}$ in the network with an error $e_{i,j}$.
\item Input nodes corresponding to an input $x_i$ of the network. These nodes do not have incoming edges.
\item Output nodes corresponding to an output $y_j$ of the network. These nodes do not have outgoing edges. 
\end{compactenum}

The edges in the networks are labeled with the weights.

For the encoding, we will have the following decision variables:
\begin{compactenum}
    \item We will introduce a continuous variable $x_i$ for each input variable in the network and $y_j$ for each output variable.
    \item For each activation function node $n_{i,j}$, we will introduce an output variable $z_{i,j}$ along with a set of temporary variables $\alpha_{i,j}$ representing its input, as well as continuous variables $t_{i,j}^{(0)}, \ldots, t_{i,j}^{(k-1)}$  and binary variables $w_{i,j}^{(0)}, \ldots, w_{i,j}^{(k-1)}$ for each ReLU term in the decomposition from Theorem~\ref{thm:relu-encoding}. Additionally, we introduce a variable $\epsilon_{i,j}$ to model the error.
\end{compactenum}

The \emph{node variable} is the corresponding input variable for an input node, its corresponding output variable for the output node, and the output variable $z_{i,j}$ for each activation function node.

We will now enumerate the constraints needed in the encoding.

\noindent\textbf{Input Constraints:} We add constraints of the form $\ell_i \leq x \leq u_i$ based on the input range to the verification problem.    

\noindent\textbf{Activation Function Constraints:} For each activation function node $n_{i,j}$, we first create an expression 
\[ \alpha_{i,j} = \sum_{ (n, n_{i,j}) \in \textsf{incoming}(n_{i,j})} \textsf{weight}(n, n_{i,j}) \times \textsf{var}(n) \,. \] 

Here $\textsf{incoming}(n)$ for a node $n$ refers to the incoming edges for that node, $\textsf{weight}(n, n_{i,j})$ refers to the edge weight for an edge $n \rightarrow n_{i,j}$ and $\textsf{var}(n)$ is the variable associated with a node $n$. This encodes the weighted combination of the values from the connected nodes that is input to the activation function node. 

Next, we consider the PWA function $\psihat(z)$ written out in terms of a weighted combination of ReLU functions as in Eq.~\eqref{eq:relu-equivalence} in Theorem~\ref{thm:relu-encoding}. Let $z_1, \ldots, z_{k-1}$ be the breakpoints. We will need to encode the following relu constraints: 

\[ t_{i,j}^{(0)} = \relu(z_1 - \alpha_{i,j}), t_{i,j}^{(1)} = \relu(\alpha_{i,j} -z_1), \ldots, t_{i,j}^{(k-1)} = \relu(\alpha_{i,j} -z_{k-1}) \,. \]

Each constraint of the form $t = \relu(\alpha - z_l) $ where $t = t_{i,j}^{(l)}$ and $\alpha = \alpha_{i,j}$ can be encoded as a series of linear inequalities involving a binary variable $w = w_{i,j}^{(l)}$ that is recalled below:
\[ t \geq \alpha, t \geq 0, t \leq \alpha + L_{i,j}(1 - w), t \leq L_{i,j} w \,, \]
wherein $\alpha_{i,j} \in [-L_{i,j}, L_{i,j}]$ are the propagated interval bounds on the inputs.

The key constraint relates the output variable of the neural network to the ReLU outputs and the error following Eq.~\eqref{eq:relu-equivalence}.
\[ z_{i,j} = \epsilon_{i,j} + \psihat(z_1) - m_0 t_{i,j}^{(0)} + m_1 t_{i,j}^{(1)} + \sum_{l=2}^{k-1} (m_l - m_{l-1}) t_{i,j}^{(l)} \,.\]

We add the constraint $-e_{i,j} \leq \epsilon_{i,j} \leq e_{i,j}$ to note that the error is bounded. 

To compute the range over an output $y_j$, we simply maximize the objective function $y_j$ to find the upper bound and minimize it to find the lower bound.

\section{Description of Benchmarks}\label{app:more-benchmarks}

\begin{enumerate}
    \item \emph{Small}: \texttt{fbessel}: $f(x) = \mathcal{J}_0(20 x)$ (1 input; 20 parameters; [1, 1] KAN). A highly oscillatory Bessel function of the first kind; the factor of $20$ compresses many oscillations into the unit interval, producing large, rapidly alternating curvature that challenges uniform PWA allocation.
 
    \item \emph{Small}: \texttt{fexp}: $f(x_1, x_2) = \exp(\sin(\pi x_1) + x_2^2)$ (2 inputs; 30 parameters; [2, 1, 1] KAN). A smooth exponential-of-additive function combining a univariate sine and a square term; it admits an exact 2-layer Kolmogorov--Arnold representation, making it a canonical sanity check for whether training recovers the true minimal compositional structure rather than a wider, over-parameterized fit.
 
    \item \emph{Medium}: \texttt{fexp4}: $f(x_1,x_2,x_3,x_4) = \exp\!\big(\tfrac{1}{2}(\sin(\pi(x_1^2 - x_2^2)) + \sin(\pi(x_3^2+x_4^2)))\big)$ (4 inputs; 260 parameters; [4, 4, 2, 1] KAN). A four-variable generalization combining two independent quadratic-difference and quadratic-sum sine interactions before an outer exponential; because each pairwise term must first be squared and summed, the function requires a deeper (3-layer) compositional structure than the two-variable benchmarks above, testing whether the allocator exploits multi-stage compositional sparsity.
 
    \item \emph{Small}: \texttt{fxy}: $f(x_1,x_2) = x_1 x_2$ (2 inputs; 60 parameters; [2, 2, 1] KAN). A pure multiplicative interaction; it is not additively separable and so cannot be represented by a depth-1 KAN, but admits the exact identity $2x_1x_2=(x_1+x_2)^2-(x_1^2+x_2^2)$, making it a minimal test of whether multiplication is recovered compositionally rather than approximated point-wise.
 
    \item \emph{Large}: \texttt{fexp100}: $f(x_1,\ldots,x_{100}) = \exp\!\big(\tfrac{1}{100}\sum_{i=1}^{100}\sin^2(\tfrac{\pi x_i}{2})\big)$ (100 inputs; 1010 parameters; [100, 1, 1] KAN). A 100-dimensional generalized-additive function; despite its high input dimensionality it depends on the inputs only through a symmetric sum of univariate terms, so a width-1 hidden layer suffices, making it the primary stress test of whether the method beats the curse of dimensionality on compositionally sparse functions.
 
    \item \emph{Medium}: \texttt{fnoise}: $f(x_1,\ldots,x_5) = 0.1\,x_1x_2 + 0.5\sin(x_3x_4) + \sin(x_5) + \mu$, with $\mu\sim\mathcal{N}(0,\sigma^2)$ and $\sigma=0.05$ (5 inputs; 1312 parameters; [5, 12, 1] KAN). A five-variable function combining a scaled multiplicative term, a sinusoidal product interaction, and a univariate sine, corrupted by additive Gaussian label noise; unlike the other benchmarks there is no noise-free ground truth to recover exactly, so this benchmark probes robustness to irreducible label noise rather than pure approximation error.
 
    \item \emph{Small}: \texttt{flegendre}: $f(x) = P_{\nu}^{m}(x)$, where $P_{\nu}^{m}(x) = (1-x^2)^{m/2}\,\frac{d^m}{dx^m}P_{\nu}(x)$ for integer $m\ge 0$ (1 input; 100 parameters; [1, 5, 1] KAN), with fixed $m=1$ and $\nu=3.0$. The five parallel splines in the hidden layer each receive the same scalar input but learn distinct nonlinear transformations (different rows of the associated Legendre matrix), so their Lipschitz constants vary --- giving the knapsack allocator meaningful differentiation signal despite the narrow single-hidden-layer structure.
 
    \item \emph{Small}: \texttt{fsphharm}: $f(x_1, x_2) = \Re\{Y_{n}^{m}(x_1,x_2)\}$, where $Y_{n}^{m}(\theta,\phi)=\sqrt{\frac{2n+1}{4\pi}\frac{(n-m)!}{(n+m)!}}\,P_{n}^{m}(\cos\phi)\,e^{im\theta}$ (2 inputs; 140 parameters; [2, 3, 2, 1] KAN), with fixed $m=1$, $n=1$. The real part of a spherical harmonic; the function is smooth globally but exhibits oscillatory angular variation whose sensitivity profile differs across the two input dimensions, providing a bivariate test of cross-path sensitivity estimation.
 
    \item \emph{Small}: \texttt{fellipeinc}: $f(x_1, x_2) = E(x_1\,|\,x_2)$, where $E(\phi\,|\,m)=\int_{0}^{\phi}\sqrt{1-m\sin^2\vartheta}\,d\vartheta$ (2 inputs; 70 parameters; [2, 2, 1, 1] KAN). An incomplete elliptic integral of the second kind; the integrand exhibits rapidly varying curvature as $m \to 1$, creating localized high-error regions where adaptive segment concentration is most beneficial.
 
    \item \emph{Medium}: \texttt{fellipkinc}: $f(x_1, x_2) = K(x_1\,|\,x_2)$, where $K(\phi\,|\,m)=\int_{0}^{\phi}\frac{d\vartheta}{\sqrt{1-m\sin^2\vartheta}}$ (2 inputs; 70 parameters; [2, 2, 1, 1] KAN). An incomplete elliptic integral of the first kind; the denominator approaches zero more aggressively than in \texttt{fellipeinc}, making this a harder approximation target with stronger curvature gradients.
 
    \item \emph{Medium}: \texttt{pinnheat}: solves the standard 2D Poisson/heat equation $u_{xx}+u_{yy}=f$ on $\Omega=[-1,1]^2$ with zero Dirichlet boundary data, where $f$ is chosen so the true solution is $u=\sin(\pi x)\sin(\pi y^2)$ (2 inputs $(x,y)$; 420 parameters; [2, 5, 1] KAN). Trained with a physics-informed loss (interior PDE residual plus boundary loss) rather than direct supervision on $u$, so the benchmark stresses a second-derivative loss landscape rather than pointwise regression.
 
    \item \emph{Large}: \texttt{acopf\_14}: learns the AC optimal power flow solution map for the IEEE 14-bus system, mapping bus-level active/reactive load inputs to bus voltages and branch power flows (22 inputs; 186 outputs; 53248 parameters; [22, 32, 186] KAN).
 
    \item \emph{Large}: \texttt{weather}: predicts short-term weather from a rolling window of hourly measurements, mapping one week of hourly features to a 6-hour-ahead forecast (168 inputs; 6 outputs; 22272 parameters; [168, 16, 6] KAN). A real-world multivariate time-series regression task with no known closed-form or compositional ground truth.
 
    \item \emph{Medium}: \texttt{prosthetic}: predicts knee abduction angle from a window of prosthetic sensor data, mapping the sensor window to predicted knee abduction angles at 5 successive time steps (50 inputs; 5 outputs; 3616 parameters; [50, 8, 4, 5] KAN). A real-world biomechanical time-series regression task, testing generalization on noisy sensor data rather than a synthetic ground-truth function.
\end{enumerate}

\fi 
\end{document}